\documentclass{article}

\usepackage{arxiv}

\usepackage[utf8]{inputenc}
\usepackage[T1]{fontenc}
\usepackage{hyperref}
\usepackage{url}
\usepackage{booktabs}
\usepackage{amsfonts}
\usepackage{nicefrac}
\usepackage{microtype}
\usepackage{graphicx}
\usepackage{natbib}
\usepackage{doi}

\usepackage{amsmath}
\usepackage{amssymb}
\usepackage{mathtools}
\usepackage{amsthm}
\usepackage{xcolor}
\usepackage{subcaption}
\usepackage{float}
\usepackage{algorithm}
\usepackage{algorithmic}
\usepackage[textsize=tiny]{todonotes}
\usepackage{cleveref}

\usepackage{amsmath,amsfonts,bm}
\usepackage{amsthm}

\usepackage{amssymb}

\def\eqref#1{equation~\ref{#1}}

\def\1{\bm{1}}

\def\mA{{\bm{A}}}
\def\mB{{\bm{B}}}

\def\mG{{\bm{G}}}

\def\mI{{\bm{I}}}

\def\mM{{\bm{M}}}

\def\mQ{{\bm{Q}}}

\def\mS{{\bm{S}}}

\def\mU{{\bm{U}}}
\def\mV{{\bm{V}}}

\def\mX{{\bm{X}}}

\def\mZ{{\bm{Z}}}

\DeclareMathAlphabet{\mathsfit}{\encodingdefault}{\sfdefault}{m}{sl}
\SetMathAlphabet{\mathsfit}{bold}{\encodingdefault}{\sfdefault}{bx}{n}

\newcommand{\R}{\mathbb{R}}

\newcommand{\Var}{\mathrm{Var}}

\newcommand{\Cov}{\mathrm{Cov}}

\DeclareMathOperator{\sign}{sign}

\usepackage{dsfont}

\usepackage{color}
\definecolor{myfavblue}{rgb}{0.05, 0.2, 0.8}
\definecolor{keywords}{RGB}{255,0,90}
\definecolor{comments}{RGB}{0,0,113}
\definecolor{red}{RGB}{160,0,0}
\definecolor{green}{RGB}{0,150,0}
\definecolor{C0}{rgb}{0.12156862745098039, 0.4666666666666667, 0.7058823529411765}  

\definecolor{myblue}{HTML}{3182bd}
\definecolor{myred}{HTML}{de2d26}

\definecolor{mydarkblue}{rgb}{0,0.08,0.45}

\theoremstyle{plain}
\newtheorem{theorem}{Theorem}[section]

\newtheorem{lemma}[theorem]{Lemma}

\theoremstyle{definition}

\newtheorem{assumption}[theorem]{Assumption}
\theoremstyle{remark}

\newcommand*{\nameA}[1]{{\emph{OLion}}}
\newcommand*{\fullname}[1]{{\emph{Orthogonal Lion}}}
\newcommand*{\codeurl}{\url{https://github.com/kv-wang/OLion}}

\title{OLion: Approaching the Hadamard Ideal by Intersecting Spectral and $\ell_{\infty}$ Implicit Biases}

\author{
  Zixiao Wang$^*$ \\
  Peking University \\
  \texttt{herowangzx@stu.pku.edu.cn} \\
  \And
  Yifei Shen$^*$ \\
  Microsoft Research Asia\\
  \texttt{yshenaw@connect.ust.hk} \\
  \And
  Huishuai Zhang\thanks{Three authors contribute equally. Correspond to Huishuai Zhang.}\\
  Peking University\\
  \texttt{zhanghuishuai@pku.edu.cn} \\
}

\renewcommand{\shorttitle}{OLion: Hadamard Ideal via Spectral and $\ell_{\infty}$ Biases}

\hypersetup{
  pdftitle={OLion: Approaching the Hadamard Ideal by Intersecting Spectral and $\ell_{\infty}$ Implicit Biases},
  pdfauthor={Authors},
  pdfkeywords={Machine Learning, Optimization, Large Language Models, NLP, Pretraining, Implicit Bias, Muon, Lion},
}

\begin{document}
\maketitle

\begin{abstract}
Many optimizers can be interpreted as steepest-descent methods under norm-induced geometries, and thus inherit corresponding implicit biases. We introduce \nameA{} (\fullname{}), which combines spectral control from orthogonalized update directions with $\ell_\infty$-style coordinate control from sign updates. \nameA{} forms a Lion-style momentum direction, approximately orthogonalizes it via a few Newton--Schulz iterations, and then applies an entrywise sign, providing an efficient approximation to taking a maximal step over the intersection of the spectral and $\ell_\infty$ constraint sets (a scaled Hadamard-like set for matrix parameters). Despite the strong nonlinearity of orthogonalization and sign, we prove convergence under a mild, empirically verified diagonal-isotropy assumption. Across large-scale language and vision training, including GPT-2 and Llama pretraining, SiT image pretraining, and supervised fine-tuning, \nameA{} matches or outperforms AdamW and Muon under comparable tuning while using only momentum-level optimizer state, and it mitigates optimizer mismatch when fine-tuning AdamW-pretrained checkpoints.
\end{abstract}


\noindent\textbf{Code:} \codeurl

\section{Introduction}
In addition to wall-clock efficiency and stability, the choice of optimizer shapes \emph{which} solutions are found through its \emph{implicit bias}~\cite{gunasekar2018characterizing,soudry2018implicit,li2023implicit,wang2021implicit,bernstein2018signsgd,vasudeva2024rich,lyu2020gradient}.
A useful unifying view is that many first-order methods implement (approximately) steepest descent under different norm-induced geometries, yielding different regularization effects even when the explicit objective is identical~\citep{bernstein2024old}.

\begin{figure}[tbp]
    \centering
    \includegraphics[width=\linewidth]{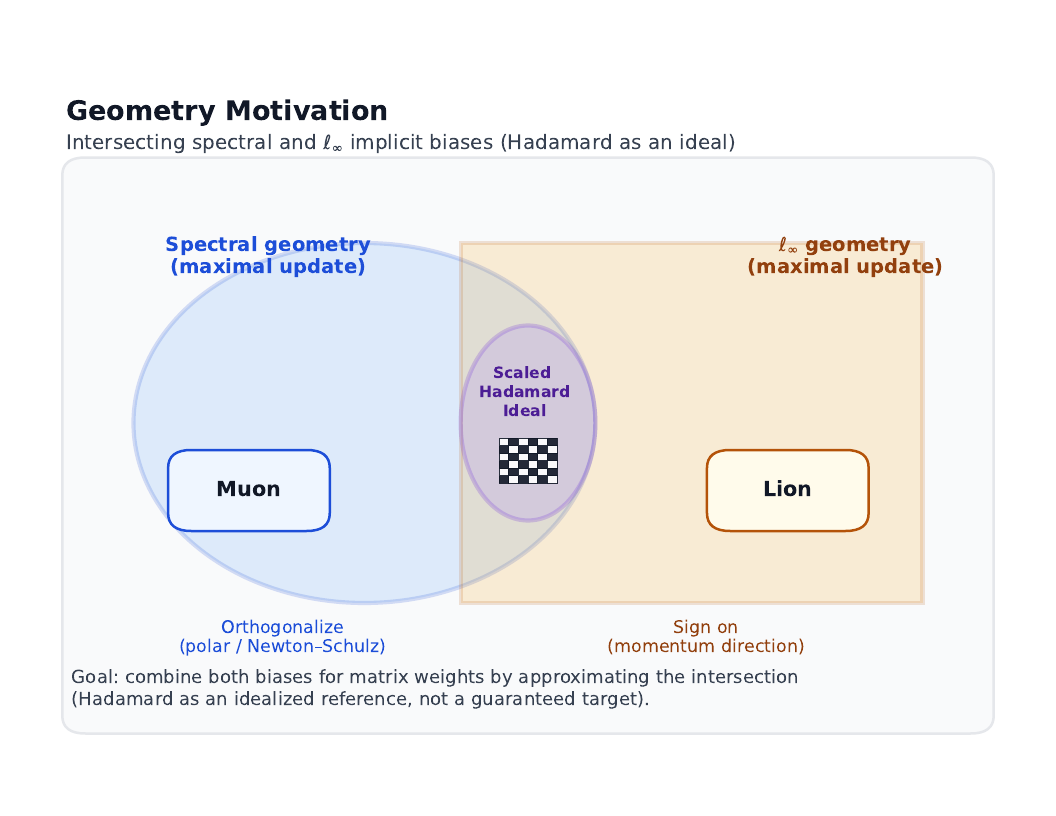}
    \caption{The geometry motivation: We view Muon and Lion as maximal-update methods under two norm-induced geometries: a spectral geometry (orthogonalization / polar factor) and an $\ell_\infty$ geometry (sign-based coordinate normalization). Their intersection suggests a scaled Hadamard set as an idealized target for matrix-shaped updates, motivating an intersection-seeking design.}
    \label{fig:geometry_motivation}
\end{figure}

Figure~\ref{fig:geometry_motivation} illustrates the geometric motivation that guides this work. Two recent optimizers highlight complementary biases. \emph{Muon} uses a structured update for matrix-shaped parameters: it orthogonalizes the momentum direction using  a small number of Newton--Schulz iterations, producing a direction-only update with a flattened singular-value profile.
From the implicit-bias viewpoint, Muon promotes a form of spectral control by limiting amplification along singular directions. In contrast, \emph{Lion} emphasizes a different geometry: it applies an element-wise sign to a momentum direction. The sign operation caps each coordinate’s contribution and behaves like steepest descent under an $\ell_\infty$ constraint.

These developments motivate a natural question.
Muon provides strong \emph{global} control through spectral normalization but lacks element-wise normalization; Lion provides strong \emph{coordinate-wise} control but does not enforce spectral structure. 
\emph{Can we obtain the benefits of both spectral and $\ell_\infty$ implicit biases in a single optimizer?} 
A related practical motivation is the \emph{pretrain--fine-tuning mismatch}: many widely used pretrained checkpoints were produced with optimizers that include element-wise normalization (e.g., AdamW-like behavior), and optimizer mismatch can measurably affect fine-tuning dynamics. For example, recent studies show that Muon fine-tuning of models pretrained with Adam performs worse than its Adam counterpart~\cite{liu2025muon}.

We approach this question through the geometry of intersecting constraints (Figure~\ref{fig:geometry_motivation}). 
For matrix parameters, Muon’s orthogonalization encourages updates that lie near a scaled orthogonal set (a spectral-norm geometry), while Lion’s sign update encourages proximity to the $\ell_\infty$ extreme set (entries with equal magnitude).
For square matrices, the intersection of these two structures corresponds to a scaled \emph{Hadamard set}, i.e., matrices with orthogonal rows/columns and entries $\pm 1/\sqrt{d}$.
Although the Hadamard ideal may not be literally solvable, the perspective is useful: composing an orthogonalization map with a sign map resembles an approximate intersection-seeking procedure.

Motivated by the view, we propose \nameA{} (\fullname{}), which integrates Muon’s Newton--Schulz orthogonalization with a Lion-style sign update.
At each step, \nameA{} forms a momentum direction, orthogonalizes this direction, and then applies an element-wise sign to obtain a coordinate-normalized update. In the implementation, we optionally apply a lightweight magnitude alignment (e.g., Root Mean Squares (RMS) scaling)  stabilizes effective step sizes across layers and tensor shapes. 
As a result, \nameA{} preserves Muon’s memory efficiency while incorporating the practical benefits of sign-based updates.

\begin{figure}[tbp]
    \centering
    \begin{subfigure}{0.48\linewidth}
        \centering
        \includegraphics[width=\linewidth]{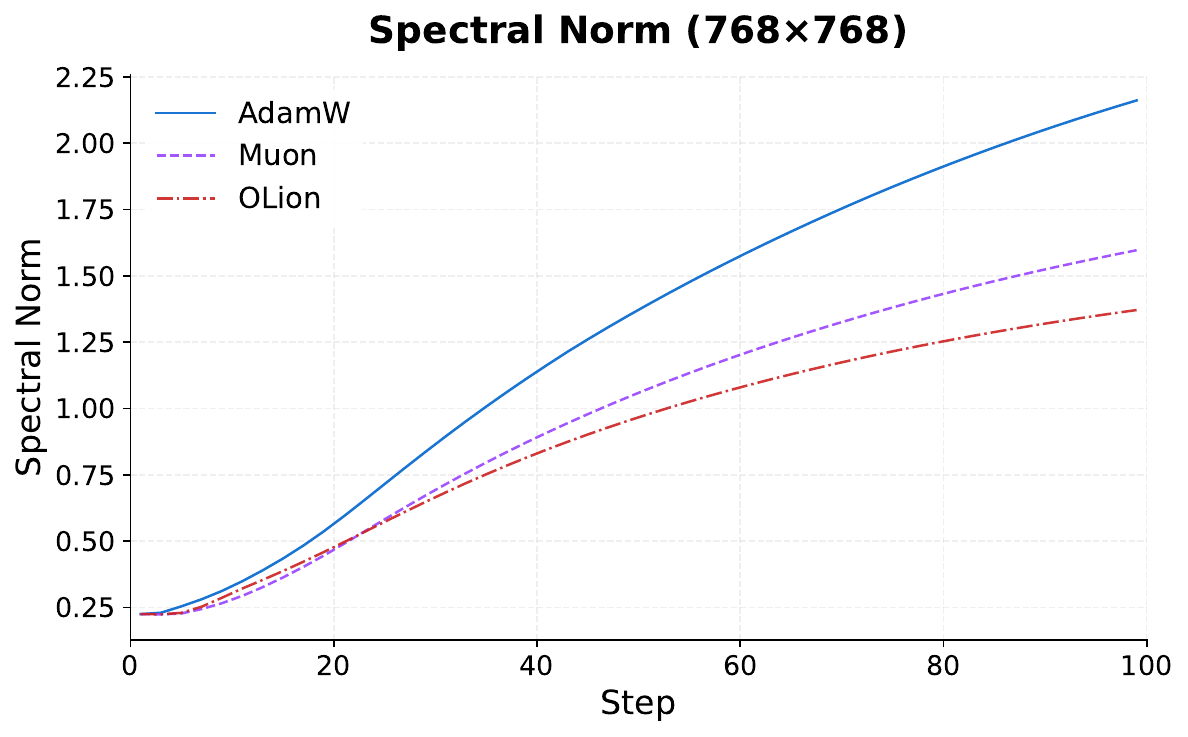}
        \caption{}
        \label{fig:subfig1}
    \end{subfigure}
    \hfill
    \begin{subfigure}{0.48\linewidth}
        \centering
        \includegraphics[width=\linewidth]{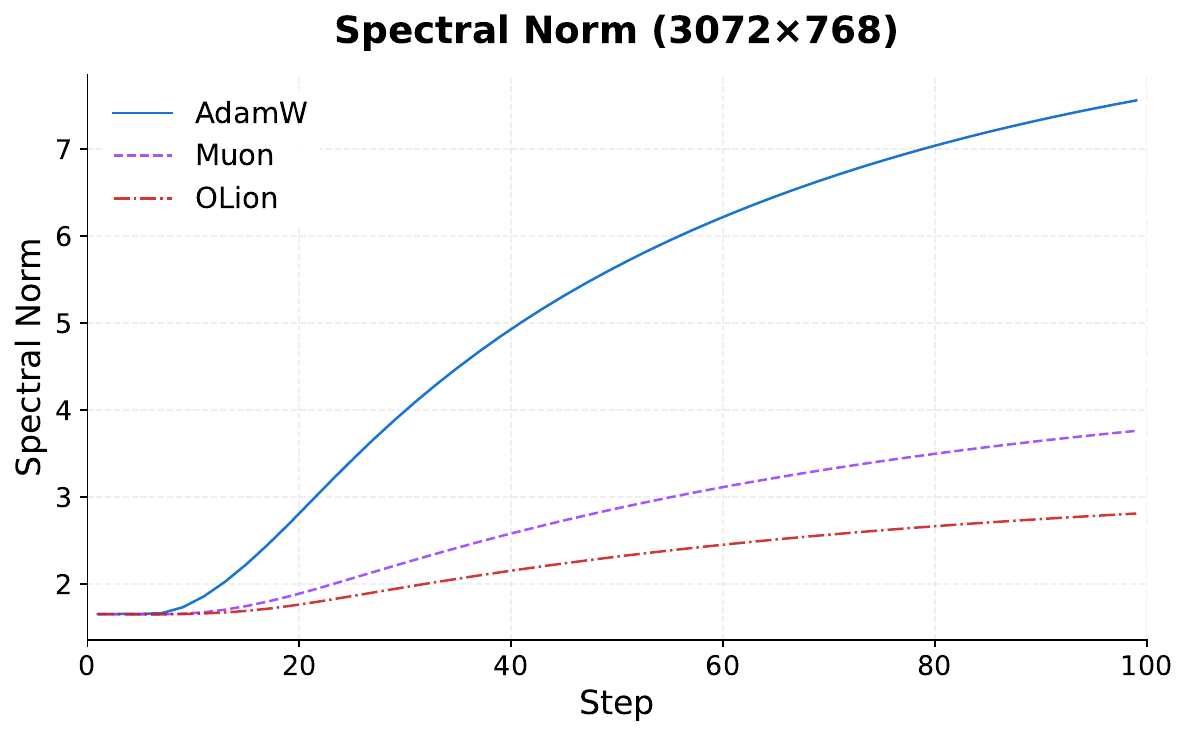}
        \caption{}
        \label{fig:subfig2}
    \end{subfigure}
    
    \vspace{0.8em}
    
    \begin{subfigure}{0.48\linewidth}
        \centering
        \includegraphics[width=\linewidth]{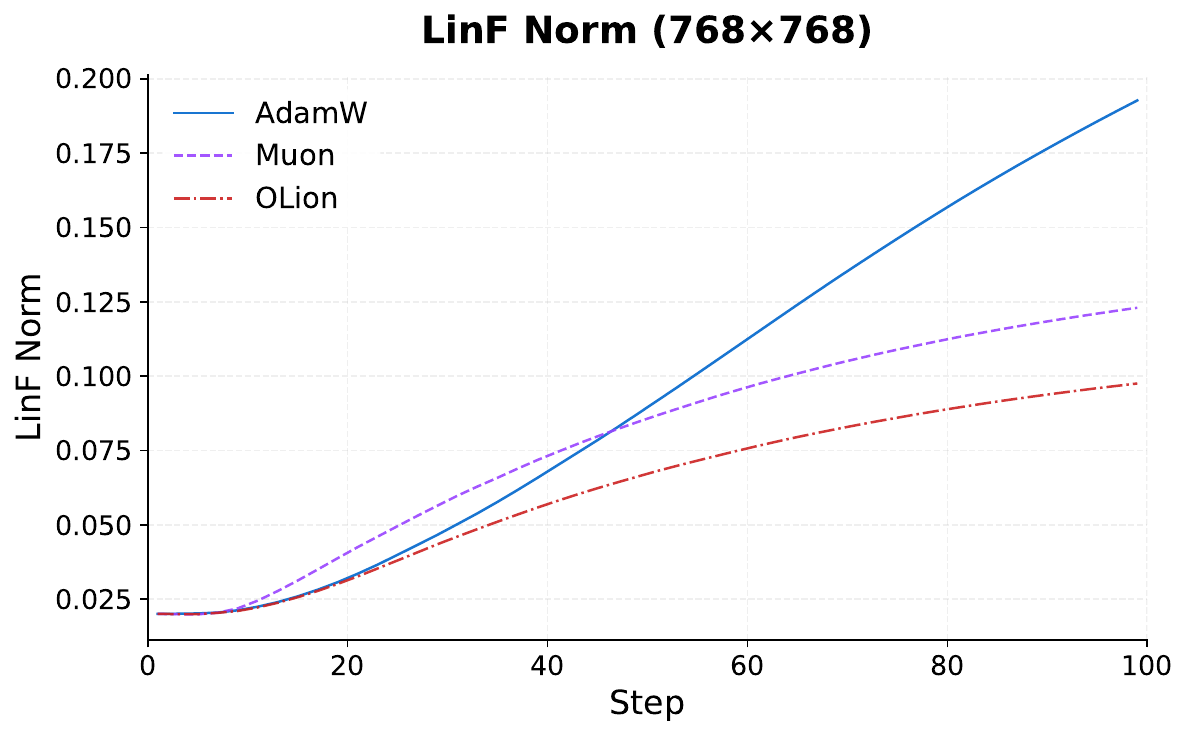}
        \caption{}
        \label{fig:subfig3}
    \end{subfigure}
    \hfill
    \begin{subfigure}{0.48\linewidth}
        \centering
        \includegraphics[width=\linewidth]{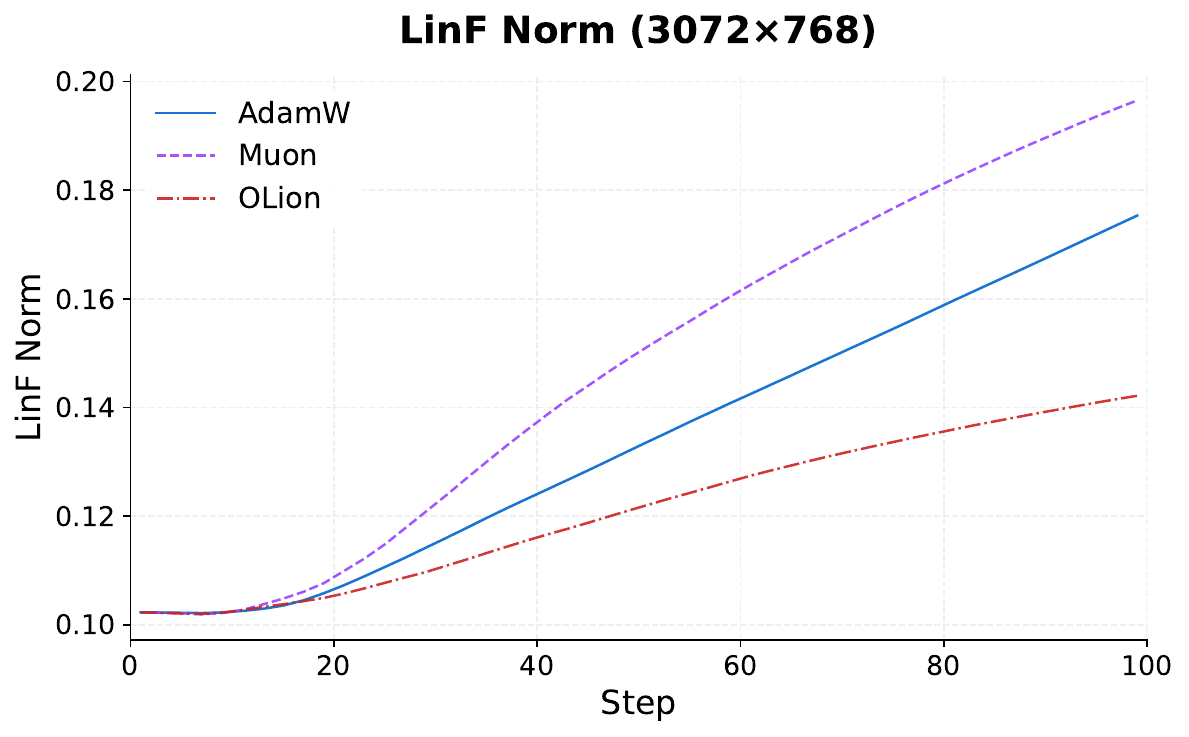}
        \caption{}
        \label{fig:subfig4}
    \end{subfigure}
    \caption{Evolution of spectral and $\ell_\infty$ norms during GPT-2 small pretraining for weight matrices of different shapes. Adam favors a small $\ell_\infty$ norm, Muon favors a small spectral norm, while \nameA{} simultaneously exhibits both implicit biases.}
    \label{fig:spectral-infty-norm}
\end{figure}


Beyond optimization, the sign operation in \nameA{} brings two practical systems-level benefits.
First, sign updates naturally support communication-efficient multi-node training: when the update direction is reduced to its element-wise sign, the communicated message can be heavily compressed (e.g., to 1-bit per entry, optionally with a shared scale), which can reduce bandwidth pressure in distributed data-parallel settings~\citep{liu2024communication}. 
Second, the $\ell_\infty$-style implicit bias induced by sign updates promotes more uniform coordinate magnitudes, leading to weight and update tensors with a flatter entrywise distribution.  Such uniformity is desirable for low-precision deployment, as quantization error is often dominated by a few large-magnitude outliers; suppressing these outliers can improve effective quantization fidelity at a fixed bit-width~\citep{chee2023quip,tseng2024quipsharp,liu2024spinquant,ashkboos2024quarot}. 

A simple experiment indeed confirms the intended bias intersection: \nameA{} maintains both a small spectral norm and a small $\ell_\infty$ norm during training, whereas other optimizers favor only one of the two, as shown in Figure~\ref{fig:spectral-infty-norm}.


\paragraph{Contributions.}
Our contributions span formulation, algorithm design, theory, and large-scale evaluation.
First, we provide a norm-geometric formulation that explicitly intersects spectral control (via orthogonalization) with $\ell_\infty$-style coordinate control (via sign), yielding a Hadamard-type ideal for matrix-shaped updates.
Second, guided by this formulation, we design \nameA{} as an effective and efficient optimizer: it applies sign-after-orthogonalization to a Lion-style momentum direction, with stable default scaling rules.
Third, we establish convergence guarantees under a new assumption, supported by theoretical justification and empirical verification.
Finally, we validate \nameA{} at scale across text and vision and across training stages, including GPT-2 and Llama pretraining, SiT image pretraining, and Llama supervised fine-tuning, showing consistent improvements over strong baselines.

\section{Related Work}
Implicit bias, i.e., how the optimizer’s geometry selects solutions among many minimizers, has been an essential lens for understanding generalization in over-parameterized models \cite{soudry2018implicit,gunasekar2018characterizing,lyu2020gradient,vasudeva2024rich,zhang2024implicitbiasadamseparable,wang2021implicit,wang2021momentum}.
Our work leverages this spirit by combining two complementary geometries \emph{within a single update}:
spectral control via orthogonalization and $\ell_\infty$-style coordinate control via sign.
While prior work has studied each geometry in isolation (and sometimes contrasted them), \nameA{} is motivated by their intersection for matrix-shaped parameters, which corresponds to a  Hadamard-type ideal.


\subsection{Spectral geometry and orthogonalized updates}
\textbf{Orthogonality-constrained optimization.}
Optimization over the Stiefel/orthogonal manifolds is a classical topic in matrix manifold optimization.
Many first-order methods proceed by taking an ambient descent step and then projecting (or retracting) back to the feasible set \cite{absil2008optimization,edelman1998geometry,wen2013feasible,gao2018framework,li2020cayley}.
A standard retraction uses the polar factor: for full-rank $\mA$, the orthogonal factor $\mathcal{O}(\mA)=\mA(\mA^\top \mA)^{-1/2}$ (equivalently $\mU\mV^\top$ from the SVD) is the closest orthogonal matrix to $\mA$ in Frobenius norm \cite{higham1986polar}.
This polar-based retraction is widely used in practical algorithms and often enables clean descent arguments under smoothness assumptions \cite{absil2008optimization,wen2013feasible,gao2018framework}.

\textbf{Optimizers as steepest descent under norm-induced geometries.}
Recent frameworks reinterpret modern optimizers as (approximate) steepest descent under different norm constraints \cite{bernstein2024old}.
Within this view, Shampoo \cite{gupta2018shampoo} can be seen as operating in a spectral geometry via preconditioning, and Muon \cite{jordan2024muon} explicitly orthogonalizes momentum updates using Newton--Schulz iterations, yielding direction-only updates.
Muon has been shown to be effective and memory-efficient for large-scale training \cite{liu2025muon}.

\textbf{Muon variants.}
Several follow-up works extend Muon in different directions.
Muon-MVR~\cite{chang2025convergence} improves stochastic convergence by adding momentum-based variance reduction.
MuonMax~\cite{goldstein2025exploration} targets robustness to learning-rate selection by aggregating norms across layers.
Turbo-Muon~\cite{boissin2025turbo} reduces the overhead of orthogonalization, while 8-bit Muon~\cite{gupta2025effective} targets optimizer-state memory. AdaMuon~\cite{si2025adamuon} augments orthogonal updates with element-wise second-moment statistics; NorMuon~\cite{li2025normuon} introduces neuron-level adaptive learning rates alongside orthogonalization. 
MuonAll~\cite{page2025muonall} removes the hybrid-training dependency by reshaping 1D parameters (e.g., biases, normalization parameters) into 2D forms so that a spectral update rule can be applied uniformly. 

\subsection{$\ell_\infty$ geometry and sign-based updates}

\paragraph{Coordinate-wise adaptivity and Adam-style methods.} Adam~\cite{kingma2014adam,loshchilov2019adamw} and related diagonal preconditioned methods (e.g., AdaGrad~\cite{duchi11a}, RMSProp) normalize gradients coordinate-wise through an adaptive diagonal scaling. In norm-geometric terms, diagonal preconditioning  can be interpreted as steepest descent under a coordinate-weighted $\ell_\infty$ geometry \cite{bernstein2024old,xie2024implicit}. 
A practical limitation is optimizer-state memory due to second-moment statistics, motivating memory-efficient variants that approximate or compress these estimates, including Adafactor~\cite{shazeer2018adafactor}, SM3~\cite{anil2019memory}, CAME~\cite{luo2023came}, AdaLomo~\cite{lv2023adalomo}, and blockwise schemes such as Adam-mini~\cite{zhang2024adam}, as well as recent progress that remove explicit second-moment state such as AdamS~\cite{zhang2025adams}.

\textbf{Sign and quantized updates.}
Sign-based updates are also a classical extreme in gradient compression and quantization. 
\textsc{signSGD}~\cite{bernstein2018signsgd} replaces gradients by their entry-wise signs, yielding strong communication savings together with convergence guarantees under standard stochastic assumptions.
The majority-vote variant improves robustness in multi-worker settings~\cite{bernstein2019signsgdmajority}.
More broadly, quantized-gradient methods and error-feedback mechanisms provide principled ways to compress updates while retaining convergence \cite{alistarh2017qsgd,karimireddy2019errorfeedback,bernstein2018signsgd,bernstein2019signsgdmajority}.

\textbf{Lion and $\ell_\infty$ implicit bias.}
The Lion  (Evolved Sign Momentum) optimizer \cite{chen2023symbolic} combines a lightweight momentum mechanism with a sign update and removes Adam’s second-moment accumulator, reducing optimizer-state memory and often improving throughput. 
Follow-up theory connects Lion to constrained optimization and $\ell_\infty$ geometry interpretations~\cite{chen2025lionsecretlysolvesconstrained} and \cite{sfyraki2025lions} unifies Lion and Muon under a stochastic Frank--Wolfe viewpoint, where Lion corresponds to an $\ell_\infty$ geometry while Muon operates on a spectral ball.
Convergence analyses for Lion in non-convex stochastic settings and distributed regimes have also been developed \cite{jiang2025convergence}, and smooth variants such as RLion replace $\operatorname{sign}(\cdot)$ with a smooth transformation to mitigate potential instability \cite{rong2025refined}.

\paragraph{Positioning of OLion.} OLion is motivated by the intersection of two geometries, and the corresponding feasible update sets for matrix-shaped parameters lead to a Hadamard-type ideal. OLion approaches such intersection in a single update while retaining a momentum-level memory footprint.

\newcommand{\ip}[2]{\left\langle #1,#2\right\rangle}
\newcommand{\tr}{\mathrm{tr}}
\newcommand{\diag}{\mathrm{diag}}
\newcommand{\sym}{\mathrm{sym}}
\newcommand{\Oop}{\mathcal{O}}
\newcommand{\normF}[1]{\left\|#1\right\|_F}
\newcommand{\normN}[1]{\left\|#1\right\|_*} 


\section{Preliminaries}
\label{sec:prelim}

\paragraph{Notation.}
We consider a matrix parameter $\mX_t\in\R^{d_1\times d_2}$ with $d_1\ge d_2$.
Let $\mG_t:=\nabla f(\mX_t)$ denote the (full) gradient at step $t$, and let
$\widetilde{\mG}_t$ denote the update signal actually used by the optimizer
(e.g., a momentum/Nesterov combination of mini-batch gradients).
We use the Frobenius inner product $\ip{\mA}{\mB}:=\tr(\mA^\top \mB)$.
For models with multiple matrix-valued parameters, we write
$\overrightarrow{\mX}:=\{\mX_{(1)},\mX_{(2)},\ldots,\mX_{(L)}\}$.
Since the update rule is applied separately to each parameter block and the analysis is also separable across matrices, we present the algorithm and theory for a single matrix for clarity; extending the statements to
$\overrightarrow{\mX}$ is straightforward but requires heavier notation.

\paragraph{Update geometry and implicit bias.}
Many first-order optimizers can be viewed as choosing an update direction according to a particular update geometry,
which in turn induces an implicit bias on the solutions reached. In this work, we focus on two geometries that are
especially natural for matrix-shaped parameters: spectral control (orthogonalized directions) and $\ell_\infty$-style
coordinate control (sign directions).

\paragraph{Two structured sets and the Hadamard ideal (tall case).}
To combine these two biases for $\mX\in\R^{d_1\times d_2}$ with $d_1\ge d_2$, we define
\begin{align}
\mathcal{A}
&:= \left\{\mX\in\R^{d_1\times d_2}:\ \mX^\top \mX=\mI_{d_2}\right\},
\\
\mathcal{B}
&:= \left\{\mX\in\R^{d_1\times d_2}:\ \mX_{ij}\in\left\{\pm \frac{1}{\sqrt{d_1}}\right\}\right\}.
\end{align}
The set $\mathcal{A}$ captures the spectral extreme (column-orthonormal updates), while $\mathcal{B}$ captures the
$\ell_\infty$ extreme (uniform-magnitude, sign-pattern updates). Their intersection
$\mathcal{A}\cap\mathcal{B}$ is the set of (scaled) partial Hadamard matrices (when they exist), i.e., matrices with
orthonormal columns and entries $\pm 1/\sqrt{d_1}$. We refer to this intersection as the \emph{Hadamard ideal}.



\section{Method: Orthogonal Lion}
\label{sec:method}

Let $\widetilde{\mG}_t\in\R^{d_1\times d_2}$ denote the matrix-shaped update signal at step $t$
(e.g., a momentum-smoothed gradient; see Algorithm~\ref{alg:olion}).
Guided by the Hadamard ideal introduced in Section~\ref{sec:prelim}, we seek an update direction that is
simultaneously close to the spectral extreme $\mathcal{A}$ and the $\ell_\infty$ extreme $\mathcal{B}$.
A natural idealized target is the maximizer of the linear score over the intersection:
\begin{equation}
\mX_t^\star \in \arg\max_{\mX\in \mathcal{A}\cap\mathcal{B}} \;\ip{\mX}{\widetilde{\mG}_t}.
\label{eq:intersection_proj}
\end{equation}
Since $\mathcal{A}\cap\mathcal{B}$ is highly nonconvex, solving Equation~(\ref{eq:intersection_proj}) exactly is intractable.
Instead, we approximate it by using alternating-projection steps
onto $\mathcal{A}$ and $\mathcal{B}$, which yields a simple and efficient update rule.

\paragraph{Projection onto $\mathcal{A}$ (Orthogonalization).}
The maximizer of $\ip{\mX}{\mZ}$ over $\mX\in\mathcal{A}$ is given by:
\begin{equation}
P_{\mathcal{A}}(\mZ)
\in \arg\max_{\mX\in\mathcal{A}} \ip{\mX}{\mZ}
\,\Longrightarrow\,
P_{\mathcal{A}}(\mZ)=\Oop(\mZ),
\label{eq:projA}
\end{equation}
where if $\mZ=\mU\boldsymbol{\Sigma}\mV^\top$ is the (thin) SVD, then $\Oop(\mZ)=\mU\mV^\top$ and
$\Oop(\mZ)^\top \Oop(\mZ)=\mI_{d_2}$.

\paragraph{Projection onto $\mathcal{B}$ (entrywise sign).}
Likewise, maximizing $\ip{\mX}{\mZ}$ over $\mX\in\mathcal{B}$ decouples entrywise and yields
\begin{equation}
P_{\mathcal{B}}(\mZ)
\in \arg\max_{\mX\in\mathcal{B}} \ip{\mX}{\mZ}
\,\Longrightarrow\,
P_{\mathcal{B}}(\mZ)=\frac{\sign(\mZ)}{\sqrt{d_1}},
\label{eq:projB}
\end{equation}
where $\sign(\cdot)$ is applied entrywise.

\subsection{The OLion Algorithm}
\label{sec:olion_alg}

\begin{algorithm}[H]
\caption{\nameA{} Optimizer}
\label{alg:olion}
\begin{algorithmic}
\REQUIRE Learning rate $\eta_t$, $\beta_1$, $\beta_2$, weight decay $\lambda$, Newton--Schulz steps $K$
\REQUIRE Initial parameters $\mathbf{X}_0$, momentum  $\mathbf{M}_0 = \mathbf{0}$
\FOR{$t = 0, 1, \ldots, T-1$}
    \STATE Compute mini-batch gradient $\mathbf{g}_t$ \hfill (for $\mathbf{G}_t=\nabla f(\mathbf{X}_t)$)
    \STATE \textbf{Momentum:} $\mathbf{M}_t = \beta_2 \mathbf{M}_{t-1} + (1-\beta_2)\mathbf{g}_t$
    \STATE \textbf{Nesterov mix:} $\widetilde{\mathbf{G}}_t = (1-\beta_1)\mathbf{g}_t + \beta_1 \mathbf{M}_t$
    \STATE \textbf{Orthogonalize:} $\mathbf{Q}_t = \textsc{NewtonSchulz}(\widetilde{\mathbf{G}}_t, K)$
    \STATE \textbf{Sign operation:} $\mathbf{S}_t = \sign(\mathbf{Q}_t)$
    \STATE \textbf{RMS alignment:} $\mathbf{D}_t = \gamma_t \mathbf{S}_t$
    \STATE \textbf{Update:} $\mathbf{X}_{t+1} = \mathbf{X}_t - \eta_t \mathbf{D}_t - \lambda\eta_t \mathbf{X}_t$
\ENDFOR
\end{algorithmic}
\end{algorithm}

\paragraph{\nameA{} direction (sign-after-orthogonalization).}
Composing two projections Equation~(\ref{eq:projA}) and Equation~(\ref{eq:projB}) gives the \nameA{} direction:
\begin{equation}
P_{\mathcal{B}}\!\left(P_{\mathcal{A}}(\widetilde{\mG}_t)\right)
\;=\;
\frac{1}{\sqrt{d_1}}\,\sign\!\left(\Oop(\widetilde{\mG}_t)\right),
\label{eq:olion_core}
\end{equation}
which can be viewed as a one-step approximation to solve Equation~(\ref{eq:intersection_proj}).
Intuitively, $\Oop(\widetilde{\mG}_t)$ enforces spectral structure by normalizing singular values, and the
subsequent sign enforces coordinate-wise normalization, producing an update direction that simultaneously reflects
both implicit biases.

For convenience and consistency with Algorithm~\ref{alg:olion}, we introduce $\mS_t := \sign(\Oop(\widetilde{\mG}_t))$. We next introduce other components in Algorithm~\ref{alg:olion}.

\paragraph{Multi-timescale momentum.}
Algorithm~\ref{alg:olion} first constructs the update signal as follows, which is the same as Lion,
\[
\mathbf{M}_t=\beta_2 \mathbf{M}_{t-1}+(1-\beta_2)\mathbf{g}_t,\qquad
\widetilde{\mathbf{G}}_t=(1-\beta_1)\mathbf{g}_t+\beta_1\mathbf{M}_t.
\]
The buffer $\mathbf{M}_t$ aggregates longer-horizon history controlled by the ``slow'' coefficient $\beta_2$,
while the mixing coefficient $\beta_1$ ensures the current mini-batch gradient $\mathbf{g}_t$ contributes
a fixed fraction $(1-\beta_1)$, preserving responsiveness.

\paragraph{Root Mean Square (RMS) alignment.} Since $\mathbf{S}_t$ has fixed-magnitude entries, implementations often apply a scalar alignment to stabilize
effective step sizes across layers and shapes. In Algorithm~\ref{alg:olion}, one option is RMS alignment:
$\mathbf{D}_t := \gamma_t \mathbf{S}_t$. 
Specifically, to maintain compatibility with Adam’s learning rate schedules, we scale the RMS norm of the final update to match Adam's empirical RMS value of $\approx 0.2$~\cite{liu2025muonscalablellmtraining}, yielding
\begin{equation}
\label{eq:adamuon_rms}
\gamma_t = \frac{0.2}{\text{RMS}({\mathbf{S}}_t)} = \frac{0.2 \sqrt{d_1d_2}}{\|{\mathbf{S}}_t\|_F} \approx 0.2,
\end{equation}
where $d_1$ and $d_2$ denote the number of rows and columns of the matrix ${\mathbf{S}}_t$, respectively.

\subsection{Convergence Analysis of OLion}
\label{sec:theory}

To establish the convergence analysis of OLion, we first introduce some standard assumptions.

\begin{assumption}[Smoothness and lower boundedness]
\label{ass:smooth}
The function $f:\mathbb{R}^{d_1\times d_2}\to\mathbb{R}$ is differentiable and
$L$--smooth with respect to $\normF{\cdot}$, i.e., for all $\mathbf{X},\mathbf{Y}$,
\begin{equation}
\label{eq:smooth}
f(\mathbf{Y})\le f(\mathbf{X})+\ip{\nabla f(\mathbf{X})}{\mathbf{Y}-\mathbf{X}}+\frac{L}{2}\normF{\mathbf{Y}-\mathbf{X}}^2.
\end{equation}
Moreover, $f$ is bounded below: $f(\mathbf{X})\ge f_{inf}$ for all $\mathbf{X}$.
\end{assumption}

Because the sign-after-orthogonalization operation is highly nonlinear, a direct descent analysis for OLion is generally impossible without leveraging additional structure in the gradient. To make the analysis tractable, we impose the following assumption on the singular-vector geometry of the update signal. Throughout the convergence proof, we assume that the $\widetilde{\mathbf{G}}_t$ used in Algorithm~\ref{alg:olion} satisfies this condition at every iteration.

\begin{assumption}[Diagonal-isotropy decomposition]
\label{assm:diag-balance}
For a rank-$r$ matrix $\mathbf{Z}$ with singular value decomposition form $\mathbf{Z}=\mathbf{U}\boldsymbol{\Sigma}\mathbf{V}^\top$, with $\mathbf{U}\in\mathbb{R}^{d_1\times r},\ 
\mathbf{V}\in\mathbb{R}^{d_2\times r},\ 
\boldsymbol{\Sigma}=\mathrm{diag}(\sigma_{1},\ldots,\sigma_{r})\succeq 0$. 
 There exists some $\varepsilon\ge 0$ such that,
\begin{equation}
\Big\|
\mathrm{diag}\big(\mathbf{U}^T\sign(\mathbf{U}\mathbf{V}^\top)\mathbf{V}\big)
-\frac{\|\mathbf{U}\mathbf{V}^\top\|_1}{r}\mathbf{1}
\Big\|_2
\ \le\
\varepsilon\,\frac{\|\mathbf{U}\mathbf{V}^\top\|_1}{\sqrt r}. \nonumber
\end{equation}
\end{assumption}

\paragraph{Intuition for diagonal isotropy.}
Let $\mathbf{S} = \sign(\mathbf{U}\mathbf{V}^\top)$. Define
\[
m_k := (\mathbf{U}^\top\mathbf{S}\mathbf{V})_{kk} = \mathbf{u}_k^\top \mathbf{S} \mathbf{v}_k = \langle \mathbf{u}_k \mathbf{v}_k^\top,\; \mathbf{S}\rangle,\quad k \in [r],
\]
which measures how strongly the sign pattern $\mathbf{S}$ correlates with the $k$-th rank-one singular-direction component $\mathbf{u}_k \mathbf{v}_k^\top$.
Note that
\[
\sum_{k=1}^r m_k = \operatorname{tr}(\mathbf{U}^\top \mathbf{S}\mathbf{V})=\langle \mathbf{U}\mathbf{V}^\top,\; \mathbf{S}\rangle=\|\mathbf{U}\mathbf{V}^\top\|_1,
\]
so the average correlation is $\bar m = \|\mathbf{U}\mathbf{V}^\top\|_1/r$.

Therefore, Assumption~\ref{assm:diag-balance} indicates that the $m_k$'s are nearly constant across $k$ (i.e., $\mathrm{diag}(\mathbf{U}^\top\mathbf{S}\mathbf{V})\approx \bar m\,\mathbf{1}$). The diagonal-isotropy assumption rules out situations where the sign pattern $\mathbf{S}$ aligns much more strongly with some singular directions than others.

We note that random matrices satisfy  Assumption~\ref{assm:diag-balance} for $\varepsilon=o(1)$ with high probability as the matrix dimension gets larger. We formally state the results in Appendix~\ref{app:random-diagonal-isotropy}. We also empirically verify that along the training trajectory, the $\widetilde{\mathbf{G}}_t$ satisfy Assumption 4.2 with small values of $\varepsilon$ (see Appendix~\ref{app:diagonal-isotropy-verification}).

Under Assumption~\ref{assm:diag-balance}, we can prove an important lemma as follows.

\begin{lemma}[Cancellation-aware upper bound]\label{lem:upper-bound}
For a rank-$r$ matrix $\mZ$ with singular value decomposition form $\mZ = \mU\boldsymbol{\Sigma}\mV^\top$ with $\mathbf{U}\in\mathbb{R}^{d_1\times r},\ 
\mathbf{V}\in\mathbb{R}^{d_2\times r},\ 
\boldsymbol{\Sigma}=\mathrm{diag}(\sigma_{1},\ldots,\sigma_{r})\succeq 0$. Suppose that Assumption~\ref{assm:diag-balance} holds. Let $\alpha :=\tr(\boldsymbol{\Sigma})/r$. Then
\[
\Big|\big\langle \mZ-\alpha \mU\mV^\top,\ \sign(\mU\mV^\top)\big\rangle\Big|
\ \le\
\varepsilon\,\frac{\|\mU\mV^\top\|_1}{\sqrt{r}}\,
\Big\|\boldsymbol{\Sigma}-\alpha \mI\Big\|_F.
\]
\end{lemma}
\begin{proof}
    The proof is deferred to Appendix~\ref{app:proof-lemma-upperbound}.
\end{proof}

We note that a naïve bound to control $\ip{\mathbf{Z}-\alpha\mathbf{Q}}{\mathbf{S}}$ is given by Cauchy-Schwarz:
\[
\big|\ip{\mathbf{Z}-\alpha\mathbf{Q}}{\mathbf{S}}\big|
\le \|\mathbf{Z}-\alpha\mathbf{Q}\|_F\,\|\mathbf{S}\|_F,
\]
but this bound ignores two key structural facts: (i) $\alpha=\tr(\boldsymbol{\Sigma})/r$ makes
$\boldsymbol{\Sigma}-\alpha\mathbf{I}$ \emph{trace} 0, so its positive and negative diagonal deviations
must cancel, and (ii) $\mathbf{S}=\sign(\mathbf{Q})$ is not an arbitrary matrix, but is tightly coupled
to $\mathbf{Q}=\Oop(\mathbf{Z})$. In fact, Lemma~\ref{lem:upper-bound} exploits these structures and the upper bound can be orders of magnitude tighter than Cauchy--Schwarz bound when $\mathbf{Q}$ is dense
(large $\|\mathbf{Q}\|_1$) and the diagonal correlations are nearly uniform.

Conceptually, the estimate shows that the ``sign-after-orthogonalization'' direction behaves almost as if $\mathbf{Z}$ were replaced by the isotropic surrogate $\alpha\mathbf{Q}$, with an error controlled by the singular-value spread $\|\boldsymbol{\Sigma}-\alpha\mathbf{I}\|_F$ and the (small) imbalance of diagonal correlations. As a result, the descent analysis depends on meaningful geometric quantities (spectral spread and sign-balance), rather than a pessimistic worst-case bound that treats $\mathbf{S}$ as arbitrary.

\paragraph{A clean core convergence analysis.}
Algorithm~\ref{alg:olion} contains several implementation components (mini-batch gradient, multi-timescale momentum,
and optional RMS alignment) that are useful in practice. However, for the purpose of exposing the main geometric idea behind \nameA{} while avoiding heavy notation overloading
(e.g., distinguishing $\mG_t$, $\mM_t$, $\widetilde{\mG}_t$, and their corresponding orthogonalized/sign variants),
we first present the convergence analysis in the deterministic full-gradient setting without momentum and stochasticity. Specifically, in the full-gradient setting, we use $\mG_t:=\nabla f(\mX_t)$ and define
\begin{flalign}
&\mQ_t := \Oop(\mG_t),\quad
\mS_t := \sign(\mQ_t), \\
&\mX_{t+1} \leftarrow \mX_t - \eta_t \mS_t.
\end{flalign}
The scalar alignment factor (e.g.\ RMS matching) in Algorithm~\ref{alg:olion} can be absorbed into the effective learning rate and we use a single $\eta_t$ for convenience.

\paragraph{Auxiliary quantities.}
Let $\mG_t$ have SVD $\mG_t=\mU_t \boldsymbol{\Sigma}_t \mV_t^\top$ with rank $r_t$ and
$\boldsymbol{\Sigma}_t=\diag(\sigma_{t,1},\ldots,\sigma_{t,r_t})$.
Define
\[
\alpha_t := \frac{\tr(\boldsymbol{\Sigma}_t)}{r_t},
\qquad
\rho_t := \frac{\|\boldsymbol{\Sigma}_t-\alpha_t \mI\|_F}{\alpha_t \sqrt{r_t}}.
\]
Define the OLion stationarity measure
\begin{equation}
\label{eq:stat_measure}
\Phi_t \;:=\; \|\mQ_t\|_1\,\alpha_t\big(1-\varepsilon\rho_t\big),
\end{equation}
which is the quantity we aim to control.

\begin{theorem}[Descent and OLion convergence]
\label{thm:descent_stationarity_simplified}
Suppose Assumption~\ref{ass:smooth} holds. At each step $t$, assume $\mG_t=\nabla f(\mX_t)$ satisfies
Assumption~\ref{assm:diag-balance} with parameter $\varepsilon$. Consider the deterministic OLion update
$\mX_{t+1}=\mX_t-\eta_t \mS_t$ with $\mS_t=\sign(\Oop(\mG_t))$.
Then for all $t$,
\begin{equation}
\label{eq:descent_phi_simplified}
f(\mX_{t+1})
\le
f(\mX_t)
-\eta_t \Phi_t
+\frac{L}{2}\eta_t^2 d_1d_2,
\end{equation}
where $\Phi_t$ is defined in Equation~(\ref{eq:stat_measure}). Consequently, for any $T\ge 1$,
\begin{equation}
\label{eq:sum_phi_simplified}
\sum_{t=0}^{T-1}\eta_t \Phi_t
\ \le\
f(\mX_0)-f_{\inf}
+\frac{L}{2}d_1d_2\sum_{t=0}^{T-1}\eta_t^2.
\end{equation}
\end{theorem}

\paragraph{Discussion and interpretation.}
Theorem~\ref{thm:descent_stationarity_simplified} provides a standard smoothness-based descent guarantee for
deterministic OLion, but with a \emph{geometry-aware} stationarity measure
\[
\Phi_t
=\|\mQ_t\|_1\,\alpha_t\big(1-\varepsilon\rho_t\big),
\qquad
\mQ_t=\Oop(\nabla f(\mX_t)).
\]
Intuitively,  the factor $\|\mQ_t\|_1$ in $\Phi_t$ captures the ``density'' of the
singular spaces of $\mQ_t$ (Hadamard-like behavior leads to a large $\ell_1$ norm), while $\alpha_t$ measures the average singular value of $\mG_t$, and the term $(1-\varepsilon\rho_t)$ accounts for the deviation from diagonal isotropy and the spread of singular values. 

To connect this OLion-specific measure to the \emph{typical} first-order stationarity criterion
$\|\nabla f(\mX_t)\|_F\to 0$, we add two mild structural conditions:
(i) a dense polar factor lower bound on $\|\mQ_t\|_1$, i.e.,  $\|\mQ_t\|_1\approx\Theta(\sqrt{rd_1d_2})$ and
(ii) a spectral-flatness condition that controls the singular-value spread of $\nabla f(\mX_t)$.
Under these conditions, $\Phi_t$ is lower bounded by a positive constant times $\|\nabla f(\mX_t)\|_F$. Theorem~\ref{thm:descent_stationarity_simplified} yields the finite-time bound
\[
\frac{1}{T}\sum_{t=0}^{T-1}\|\nabla f(\mX_t)\|_F
\ \le\
\frac{\sqrt{L}\big(f(\mX_0)-f_{\inf}\big)}{c_\star\,\sqrt{T}},
\]
for some constant $c_\star$, which forces the average gradient norm to vanish at rate $O(1/\sqrt{T})$, matching typical non-convex bound.

Finally, we note that the analysis for the full gradient can be extended to incorporate (i) momentum/Nesterov mixing (treating
$\widetilde{\mG}_t$ as a biased gradient surrogate controlled by $(\beta_1,\beta_2)$) and (ii) stochastic gradients
(replacing $\mG_t$ by an unbiased estimator and handling the additional variance terms) using standard techniques
from nonconvex optimization~\cite{wang2023closing,wang2024provable,li2023convergence}. Such extensions are conceptually routine but lead to substantially denser presentations.

\section{Experiments}
We evaluate \nameA{} against AdamW and Muon across language and vision, and in both pretraining and supervised fine-tuning settings.
Within each setting, we hold the architecture, data, and training pipeline fixed and vary only the optimizer and its hyperparameters.
Unless otherwise stated, we use standard defaults for AdamW and Muon and our default $(\beta_1,\beta_2)$ for \nameA{}; full configuration details are provided in Appendix~\ref{app:exp_details}.

\subsection{\nameA{} performs well on LM pretraining}
\label{sec:exp_lm_pretrain}

We evaluate \nameA{} on standard large-scale language-model pretraining setups.
To ensure a fair comparison, we (i) use each baseline optimizer with its \emph{default} hyperparameter
configuration as recommended in prior work, and (ii) match the overall update magnitude across methods
so that improvements are attributable to the \emph{update geometry} rather than trivially larger steps.
Concretely, we first pretrain GPT-2 models (124M--770M) for 48B tokens (100K iterations) following the
default AdamW configuration, and then train Llama-2-7B for 32B tokens (8K iterations with batch size 4M)
using the default hyperparameters for this regime.

\paragraph{GPT-2 pretraining (124M--770M).}
We study scalability on GPT-2 models trained on OpenWebText at three sizes: \emph{small} (124M), \emph{medium} (355M),
and \emph{large} (770M).
Figure~\ref{fig:gpt2_pretraining} reports training-loss trajectories.
Across all model sizes, \nameA{} converges faster than both AdamW, Lion and Muon.
This suggests that combining spectral control (via orthogonalization) with coordinate-wise control
(via the sign update) improves optimization speed in this canonical pretraining setting.

\begin{figure*}[htbp]
    \centering
    \begin{subfigure}{0.32\textwidth}
        \centering
        \includegraphics[width=\linewidth]{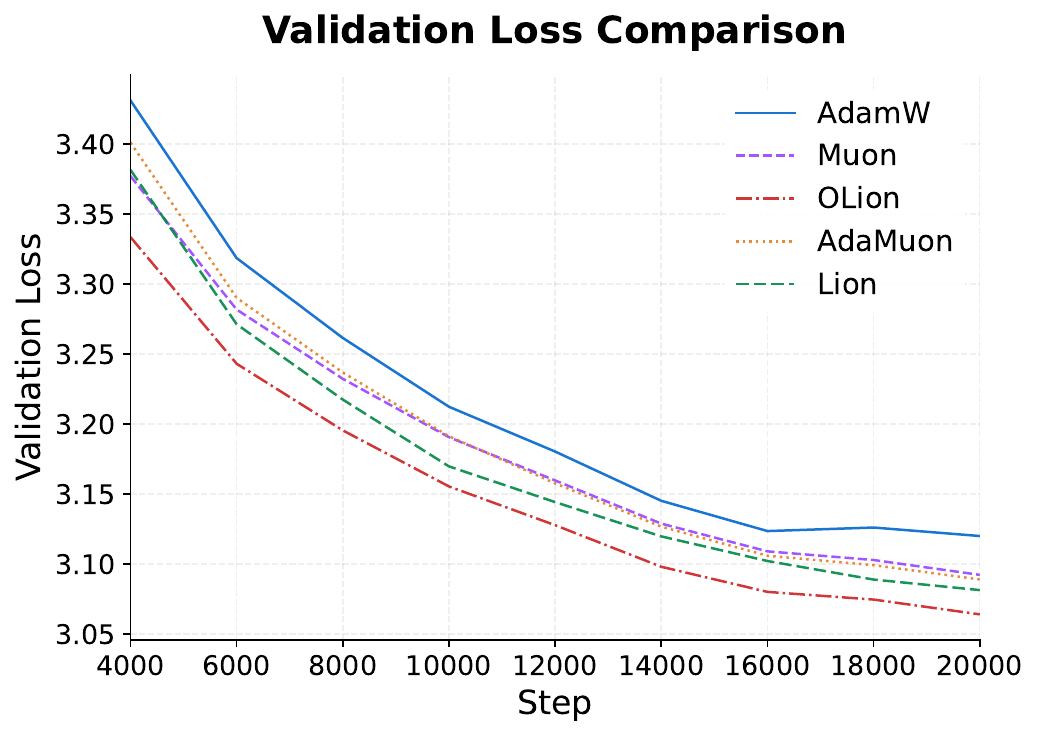}
        \caption{GPT-2 \emph{small} (124M).}
    \end{subfigure}
    \begin{subfigure}{0.32\textwidth}
        \centering
        \includegraphics[width=\linewidth]{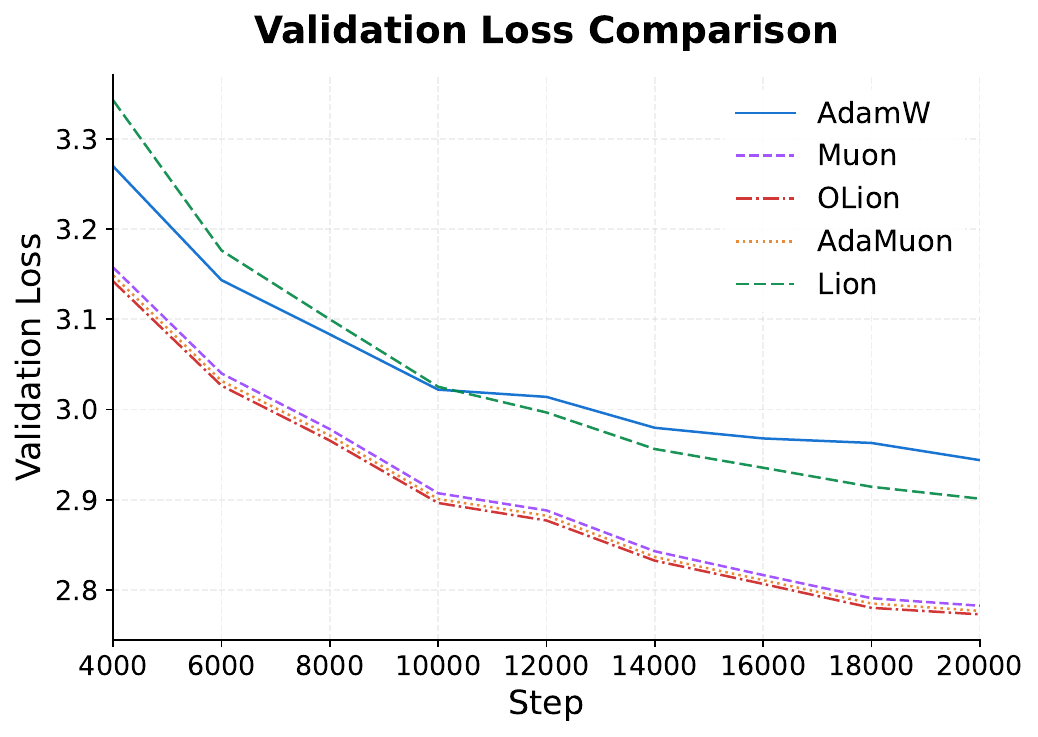}
        \caption{GPT-2 \emph{medium} (355M).}
    \end{subfigure}
    \begin{subfigure}{0.32\textwidth}
        \centering
        \includegraphics[width=\linewidth]{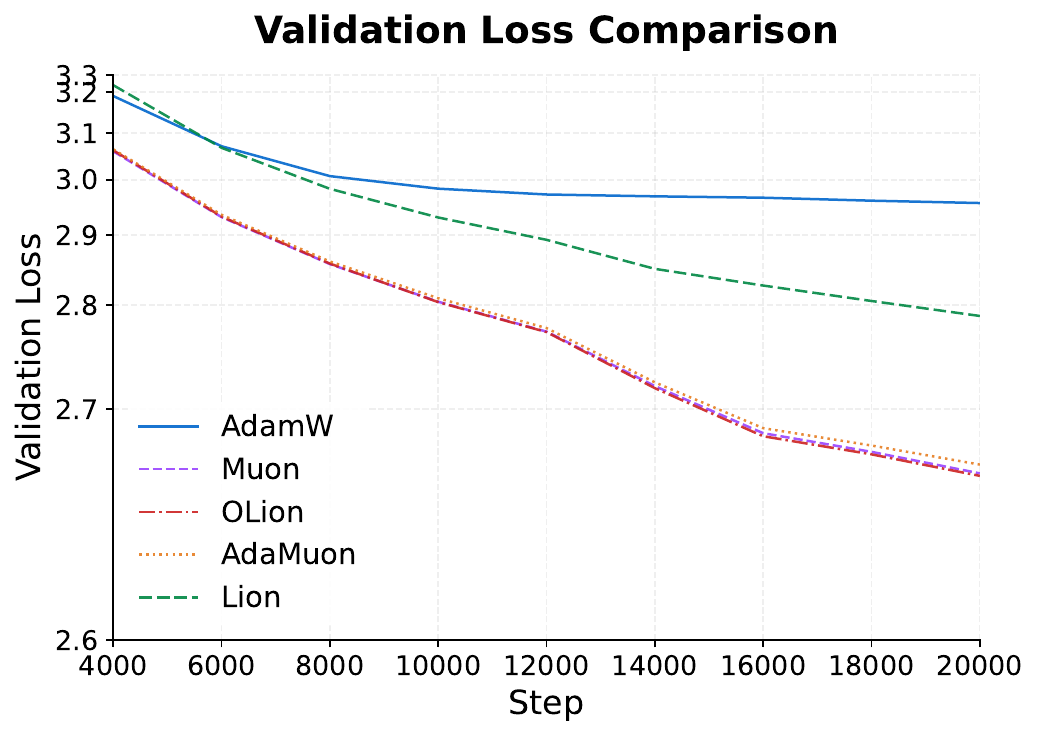}
        \caption{GPT-2 \emph{large} (770M).}
    \end{subfigure}
    \caption{Validation losses for GPT-2 pretraining with AdamW, Lion, Muon, and \nameA{}.
    \nameA{} converges faster across all model sizes.}
    \label{fig:gpt2_pretraining}
\end{figure*}

\paragraph{Llama-2-7B pretraining.}
We next evaluate billion-parameter pretraining with Llama-2-7B by using the FSDP training pipeline.
Figure~\ref{fig:llama7b_pretraining} shows both training and validation loss.
\nameA{} maintains consistently lower loss throughout training compared to Muon, Lion and AdamW,
demonstrating that the gains observed on GPT-2 persist at the 7B scale under distributed training.

\begin{figure}[htbp]
    \begin{subfigure}{0.48\linewidth}
        \centering
        \includegraphics[width=\linewidth]{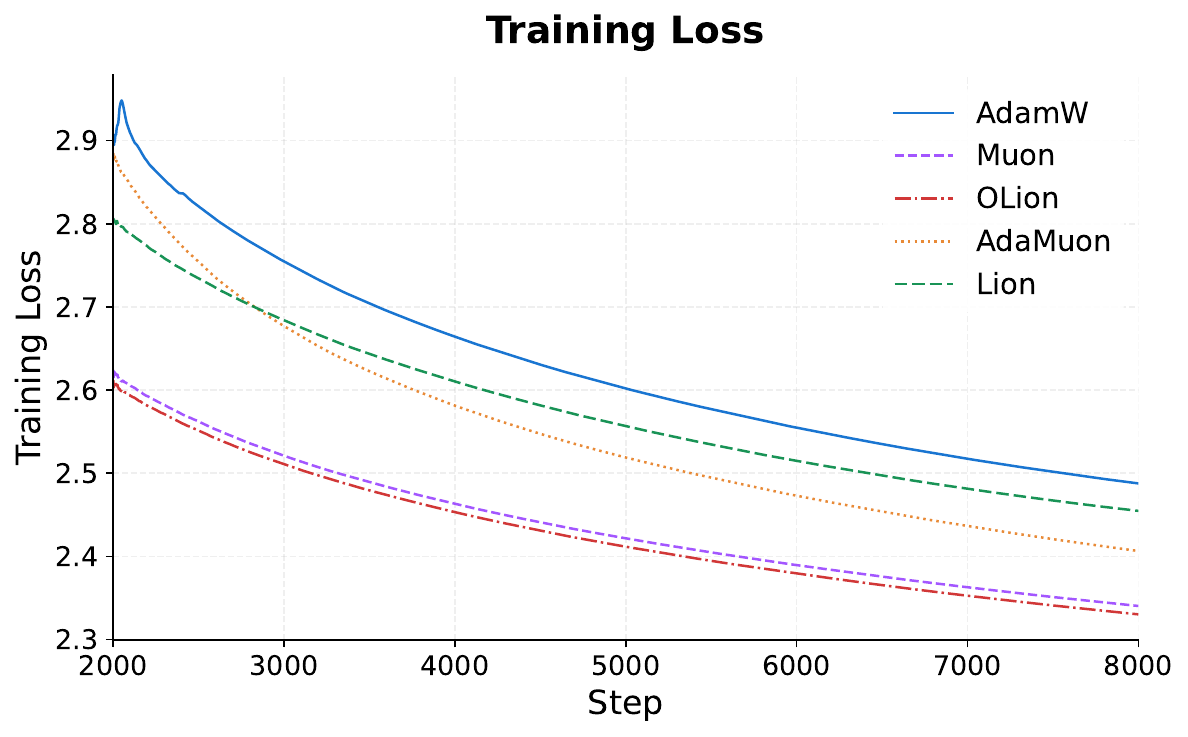}
        \caption{Training loss.}
    \end{subfigure}
    \begin{subfigure}{0.48\linewidth}
        \centering
        \includegraphics[width=\linewidth]{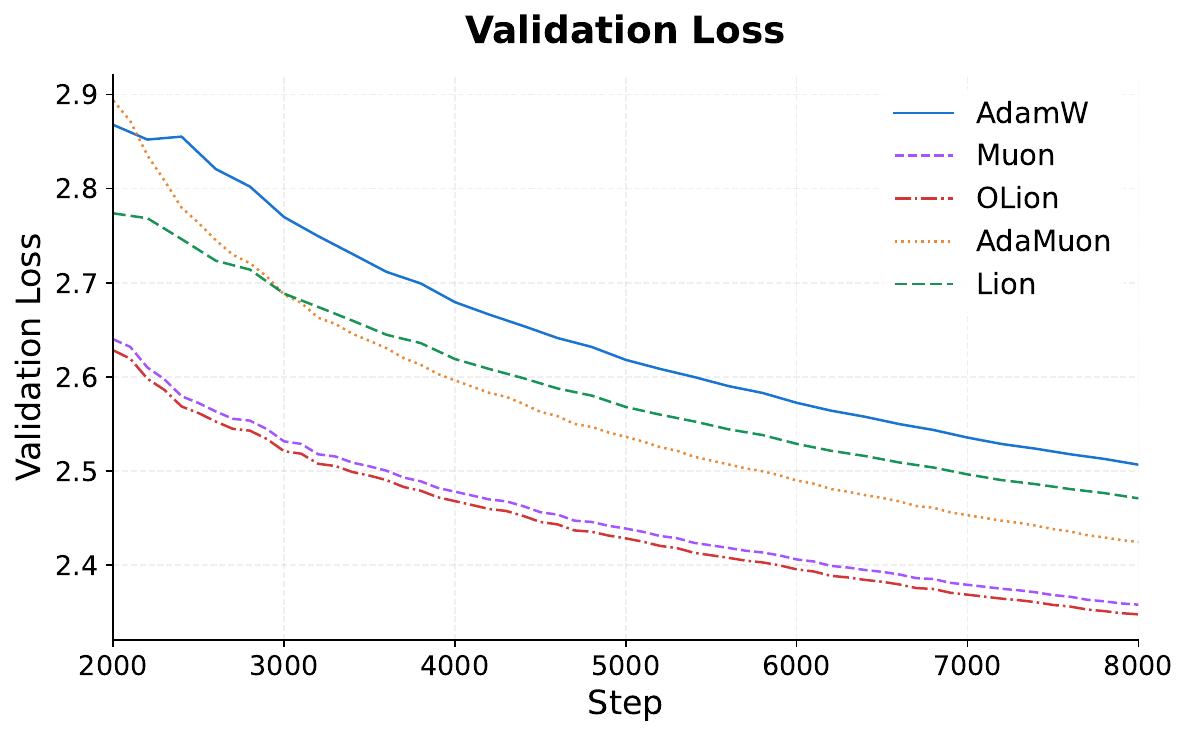}
        \caption{Validation loss.}
    \end{subfigure}
    \caption{Llama-2-7B pretraining curves with different optimizers: AdamW, Lion, Muon, and \nameA{}.}
    \label{fig:llama7b_pretraining}
\end{figure}

\subsection{\nameA{} enjoys a wide learning-rate range}
\label{sec:lr_robustness}

We further examine how sensitive \nameA{} is to the choice of learning rate.
Using the GPT-2 small pretraining setup from Section~\ref{sec:exp_lm_pretrain}, we vary \emph{only} the learning rate
and keep all other hyperparameters and training procedures fixed.
We compare \nameA{} against Muon, Lion and AdaMuon under four learning rates:
$3\times 10^{-4}$, $1\times 10^{-3}$, $2\times 10^{-3}$, and $5\times 10^{-3}$.
For each optimizer--learning-rate pair, we report the validation loss at training step 10{,}000.

As shown in Figure~\ref{fig:lr_robustness}, \nameA{} achieves consistently lower validation loss than other optimizers across the entire range of learning rates. This indicates that the improvements from \nameA{} are not tied to
a narrowly tuned step size; instead, its advantage persists across substantially different learning-rate regimes.

\begin{figure}[H]
    \centering
    \includegraphics[width=0.7\linewidth]{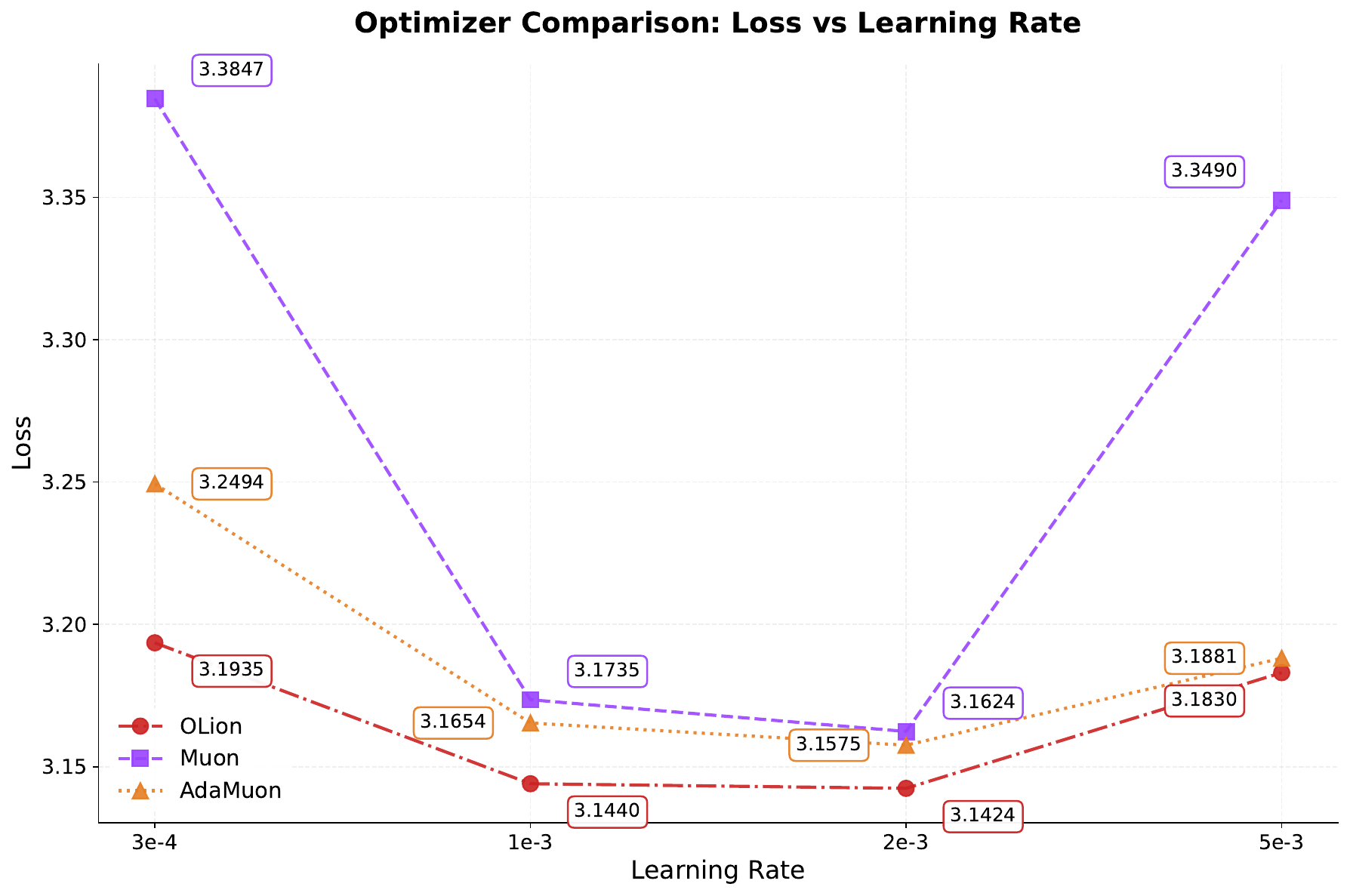}
    \caption{Validation loss at step 10{,}000 for GPT-2 small trained with different optimizers under four learning rates:
    $3\times 10^{-4}$, $1\times 10^{-3}$, $2\times 10^{-3}$, and $5\times 10^{-3}$.
    \nameA{} consistently outperforms both baselines across all learning rates, demonstrating robustness to step-size selection.}
    \label{fig:lr_robustness}
\end{figure}

\subsection{Validating the induced implicit bias of \nameA{}}
\label{sec:exp_bias}

To test whether \nameA{} induces the implicit biases suggested by our formulation, we track both
\emph{spectral-} and \emph{$\ell_\infty$-}related statistics of weight matrices during training and at convergence.
We focus on GPT-2 small for a controlled study and compare AdamW, Muon, Lion, and \nameA{} throughout.

\paragraph{Biases along the training trajectory.}
Figure~\ref{fig:spectral-infty-norm} reports the evolution of the spectral norm and the $\ell_\infty$ norm
(maximum absolute entry) for representative weight matrices of different shapes.
Across all shown matrices, \nameA{} maintains smaller spectral norms than AdamW and Lion, and is competitive with (often better than) Muon, consistent with the effect of the orthogonalization step on the update geometry.
At the same time, \nameA{} yields systematically smaller $\ell_\infty$ norms than Muon (and typically also AdamW and Lion), indicating stronger suppression of coordinate outliers---as expected from the post-orthogonalization sign operation.
Together, these trends provide direct evidence that \nameA{} combines spectral control and coordinate-wise control during optimization.

\paragraph{Bias at the end of training.}
To probe the end-of-training behavior more directly, Figure~\ref{fig:implicit_bias} summarizes the distributions of
(i) singular values (SVD) and (ii) entrywise magnitudes (Abs) for the same set of matrices at convergence.
Relative to AdamW, \nameA{} shifts the singular-value spectrum downward, indicating a smaller effective spectral scale,
while relative to Muon, \nameA{} tightens the tail of absolute values, indicating smaller coordinate magnitudes.
These distributional shifts match the intended combination of spectral and $\ell_\infty$ implicit biases.
For additional per-layer and per-matrix case studies that further corroborate these trends, see Appendix~\ref{app:matrix_cases}.



\begin{figure}[htbp]
  \centering
  \begin{subfigure}[t]{0.98\linewidth}
    \centering
    \includegraphics[width=\linewidth]{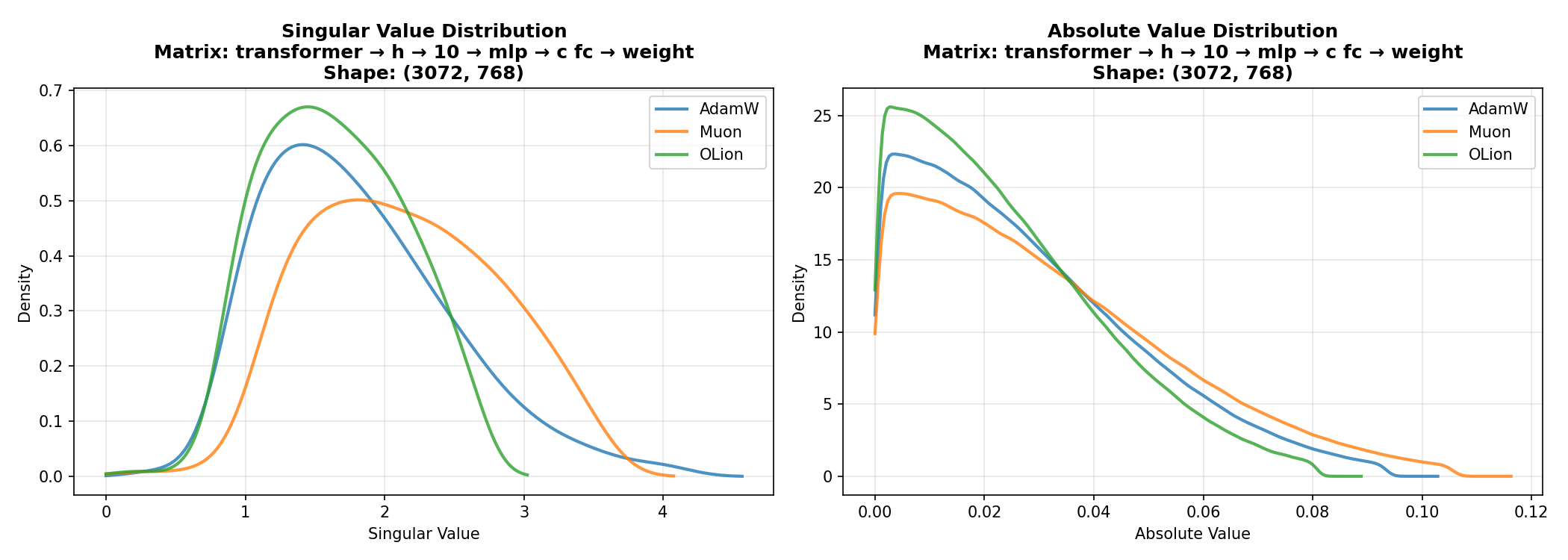}
    \caption{MLP projection (layer 10), shape $3072\times 768$.}
    \label{fig:implicit_bias_matrix_h10_mlp}
  \end{subfigure}
  \hfill
  \begin{subfigure}[t]{0.98\linewidth}
    \centering
    \includegraphics[width=\linewidth]{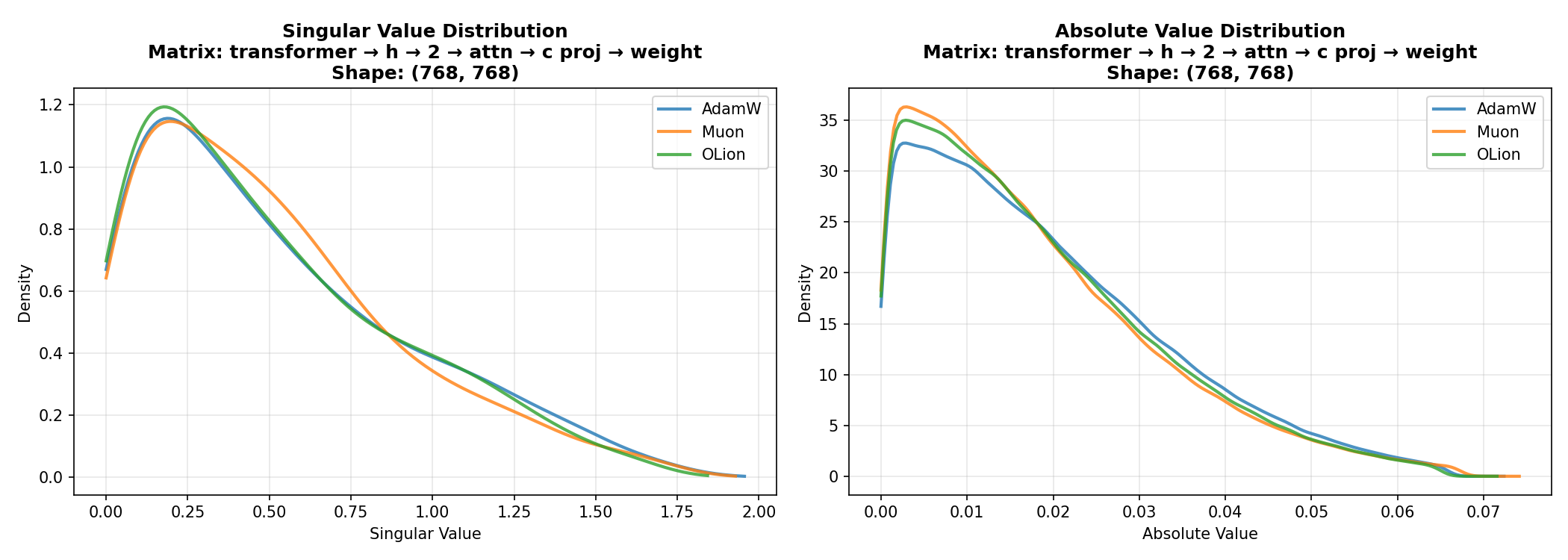}
    \caption{Attention projection (layer 2), shape $768\times 768$.}
    \label{fig:implicit_bias_h2_attn}
  \end{subfigure}
  \caption{End-of-training implicit bias on GPT-2 \emph{Small}: distributions of singular values (SVD) and entrywise magnitudes (Abs) for representative weight matrices.}
  \label{fig:implicit_bias}
\end{figure}

\subsection{Extension to image pretraining: SiT}
\label{sec:exp_sit}

We next evaluate whether the benefits of \nameA{} extend beyond language modeling to image-generation pretraining.
Specifically, we pretrain SiT-B/2 (a diffusion transformer) on ImageNet-1K with $256\times256$ resolution using the same
training recipe across optimizers.

Figure~\ref{fig:sitb2_pretraining} reports two standard objectives for this setting: the projection loss and the
denoising loss. Across both metrics, \nameA{} reduces loss faster than Muon throughout training.
These results mirror the trends observed in language-model pretraining and suggest that the advantage of intersecting
spectral and $\ell_\infty$ implicit biases is not text-specific, but carries over to diffusion-model pretraining.

\begin{figure}[htbp]
    \centering
    \begin{subfigure}{0.48\linewidth}
        \centering
        \includegraphics[width=\linewidth]{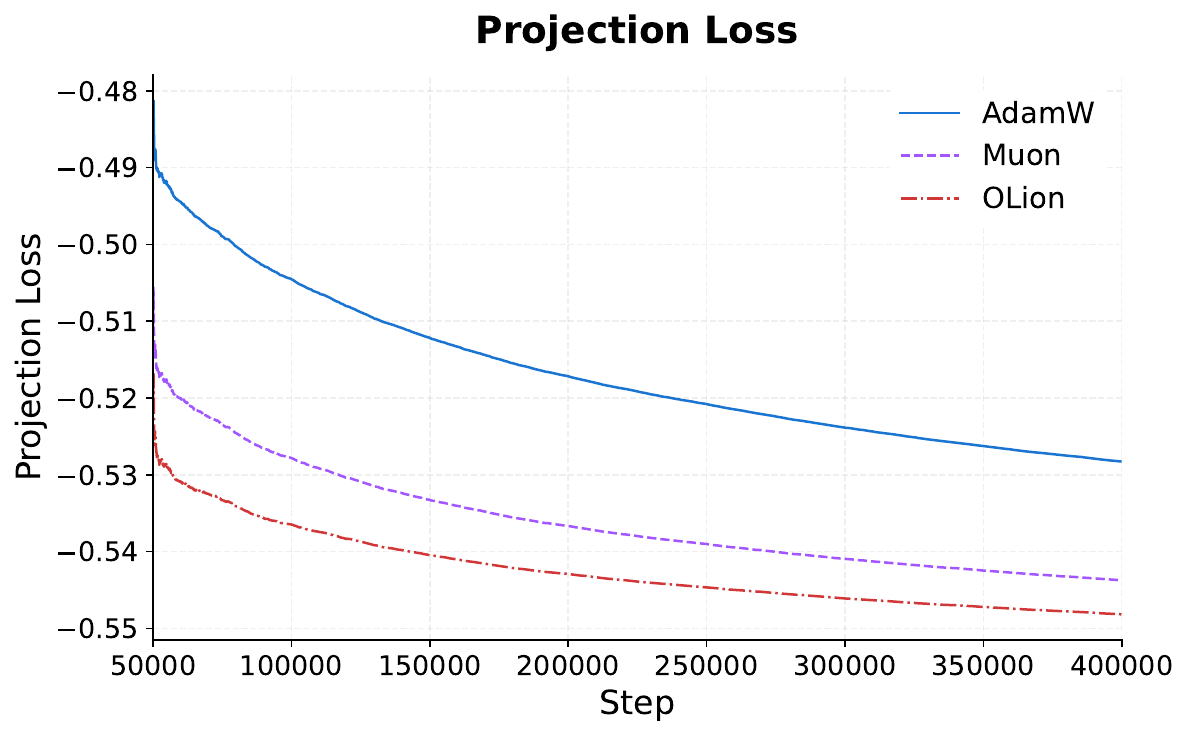}
        \caption{Projection loss.}
    \end{subfigure}
    \hfill
    \begin{subfigure}{0.48\linewidth}
        \centering
        \includegraphics[width=\linewidth]{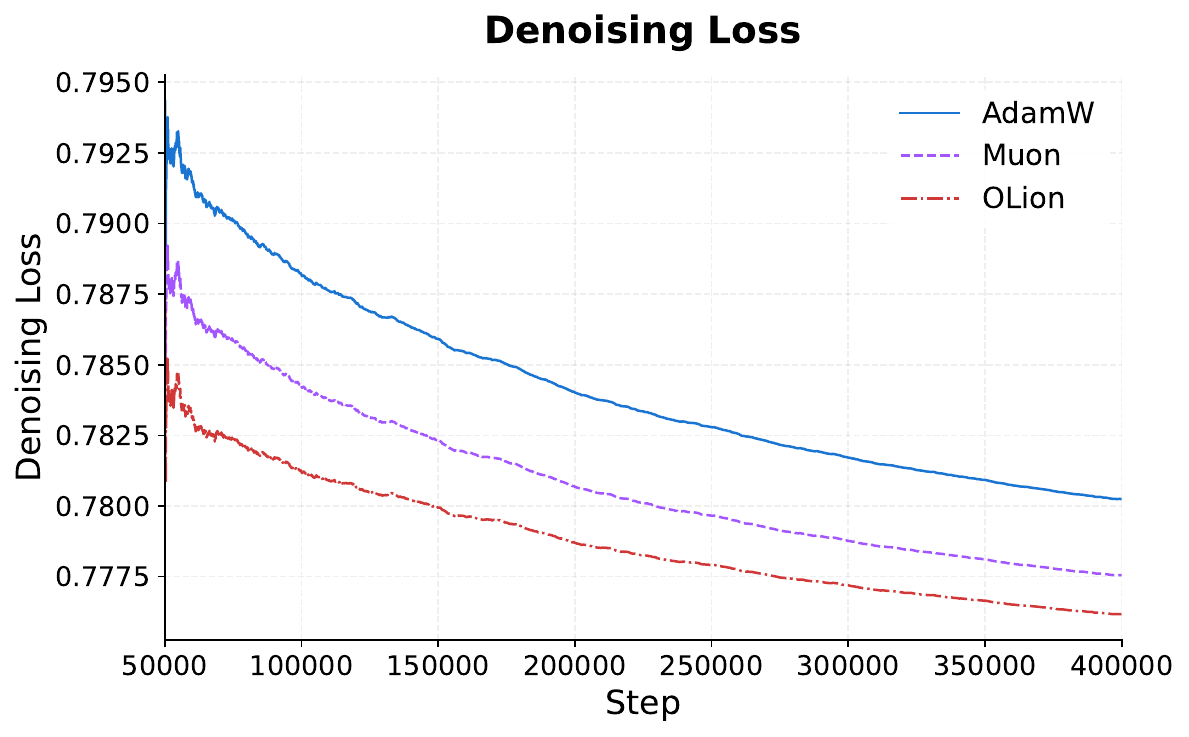}
        \caption{Denoising loss.}
    \end{subfigure}
    \caption{SiT-B/2 pretraining on ImageNet-1K at $256\times256$. \nameA{} converges faster than others on  projection/denoising losses.}
    \label{fig:sitb2_pretraining}
\end{figure}

\subsection{Extension to supervised fine-tuning: Llama-3.1-8B} \label{sec:exp_sft}

Finally, we evaluate \nameA{} on supervised fine-tuning (SFT) of Llama-3.1-8B checkpoints that were
\emph{pretrained with AdamW}. This setting is particularly relevant because prior observations suggest that Muon can
suffer from a noticeable \emph{pretrain--SFT mismatch} when applied to AdamW-pretrained models, leading to degraded downstream performance relative to AdamW~\cite{liu2025muonscalablellmtraining}.

Table~\ref{tab:sft_math} summarizes downstream benchmark accuracies after SFT;
\nameA{} performs best across the listed tasks. 

One plausible explanation is that \nameA{} reduces optimizer mismatch relative to Muon while still retaining spectral control. In particular, the post-orthogonalization sign step introduces an element-wise normalization structure that is closer in spirit to AdamW-style coordinate control, which may make the AdamW-pretrained checkpoint more compatible with \nameA{} than with Muon’s purely orthogonal (and less coordinate-normalized) update.  The detailed fine-tuning setting and training curves are provided in Appendix~\ref{app:sft}. 

\begin{table}[t]
\centering
\caption{Zero- and few-shot math benchmark accuracy (\%) for the base model and SFT models fine-tuned with different optimizers. Best results are in \textbf{bold}; second-best results are in \textbf{\textcolor{gray}{bold gray}}.}
\label{tab:sft_math}
\resizebox{\columnwidth}{!}{%
\begin{tabular}{lcccc}
\toprule
Benchmark & Base (no SFT) & AdamW SFT & Muon SFT & \nameA{} SFT \\
\midrule
GSM8K 0-shot    & 17.79 & {\color{gray}\textbf{57.99}} & 57.24 & \textbf{60.04} \\
GSM8K 4-shot    & 52.50 & {\color{gray}\textbf{60.35}} & 58.30 & \textbf{60.58} \\
MATH 0-shot     & 14.20 & 18.46 & {\color{gray}\textbf{19.32}} & \textbf{19.80} \\
NumGLUE 0-shot  & 25.70 & 35.04 & {\color{gray}\textbf{40.78}} & \textbf{42.14} \\
NumGLUE 4-shot  & 38.09 & 39.82 & {\color{gray}\textbf{43.37}} & \textbf{44.04} \\
SimulEq 0-shot  & 12.20 & {\color{gray}\textbf{21.48}} & 21.20 & \textbf{22.57} \\
Aqua 0-shot     & 23.62 & {\color{gray}\textbf{42.52}} & 41.34 & \textbf{46.45} \\
\bottomrule
\end{tabular}}
\end{table}

\section{Conclusion}
We presented \nameA{} (\fullname{}), a memory-efficient optimizer that combines spectral control from Muon-style orthogonalized updates with $\ell_\infty$-style coordinate control from Lion-style sign updates.
Our method is motivated by a maximal-update view under norm-induced geometries: orthogonalization corresponds to a spectral geometry, while sign corresponds to an $\ell_\infty$ geometry.
Composing these operations yields an efficient, single-step approximation to intersection seeking between the two geometries, with the scaled Hadamard set providing a useful idealized reference for matrix-shaped parameters.

On the theoretical side, we established that sign-after-orthogonalization preserves directional alignment with the orthogonalized direction, and provided standard convergence guarantees under non-convex smoothness conditions.
Empirically, \nameA{} consistently matches or outperforms AdamW and Muon across a broad range of regimes, including GPT-2 and Llama-2 language-model pretraining, SiT image pretraining, and supervised fine-tuning of Llama-3.1-8B, while reducing optimizer mismatch relative to Muon.

Several directions remain open. First, the Hadamard-intersection viewpoint suggests richer operator-splitting or multi-step intersection solvers beyond a single composed projection.
Second, the sign update may offer additional benefits in distributed and low-precision settings, motivating a deeper study of communication efficiency and quantization behavior under \nameA{}.
We hope \nameA{} provides a practical and principled step toward combining complementary implicit biases for large-scale training.

\section*{Impact Statement}
This paper proposes a new optimization method that improves training stability and efficiency for large-scale models. Our goal is to advance the optimization toolkit for modern machine learning, with potential positive impacts including reduced training compute for comparable performance and improved robustness across hyperparameter regimes, which may lower the barrier to reproducing and extending large-model training.

We do not introduce new data sources, collect personal data, or release any trained models or datasets beyond standard public resources used in the community. We will provide implementation details to support reproducibility, and we encourage responsible use consistent with existing best practices for large-model development and deployment.

\bibliographystyle{unsrtnat}
\bibliography{optimization}

@book{absil2008optimization,
  author    = {Absil, P.-A. and Mahony, Robert and Sepulchre, Rodolphe},
  title     = {Optimization Algorithms on Matrix Manifolds},
  publisher = {Princeton University Press},
  year      = {2008}
}

@inproceedings{liu2024communication,
  title={Communication Efficient Distributed Training with Distributed Lion},
  author={Liu, Bo and Wu, Lemeng and Chen, Lizhang and Liang, Kaizhao and Zhu, Jiaxu and Liang, Chen and Krishnamoorthi, Raghuraman and Liu, Qiang},
  booktitle={Advances in Neural Information Processing Systems (NeurIPS)},
  year={2024}
}

@inproceedings{chee2023quip,
  title={QuIP: 2-Bit Quantization of Large Language Models via Incoherence Processing},
  author={Chee, Jerry and Cai, Yaohui and Kuang, Volkan and De Sa, Christopher},
  booktitle={Advances in Neural Information Processing Systems},
  year={2023}
}

@article{ashkboos2024quarot,
  title={QuaRot: Outlier-Free 4-Bit Inference in Rotated LLMs},
  author={Ashkboos, Saleh and Hudobivnik, Amir and Babakniya, Shervin and Klafenboeck, Elias and Hoefler, Torsten},
  journal={arXiv preprint arXiv:2404.00456},
  year={2024}
}

@article{liu2024spinquant,
  title={SpinQuant: LLM Quantization with Learnable Rotations},
  author={Liu, Zechun and Oguz, Barlas and Zhao, Shangyuan and Gholami, Amir and Won, Jiyuan and Yang, Linjie and Kozareva, Zornitsa and Xiao, Ting and Ma, Xiaofei},
  journal={arXiv preprint arXiv:2405.16406},
  year={2024}
}

@article{tseng2024quipsharp,
  title={QuIP\#: QuIP with Lattice Codebooks and Weight Rescaling},
  author={Tseng, Albert and Chee, Jerry and Sun, Christopher and Cai, Yaohui and De Sa, Christopher},
  journal={arXiv preprint arXiv:2402.04396},
  year={2024}
}

@article{edelman1998geometry,
  author  = {Edelman, Alan and Arias, Tomas A. and Smith, Steven T.},
  title   = {The Geometry of Algorithms with Orthogonality Constraints},
  journal = {SIAM Journal on Matrix Analysis and Applications},
  volume  = {20},
  number  = {2},
  pages   = {303--353},
  year    = {1998},
  doi     = {10.1137/S0895479895290954}
}

@article{higham1986polar,
  author  = {Higham, Nicholas J.},
  title   = {Computing the Polar Decomposition---with Applications},
  journal = {SIAM Journal on Scientific and Statistical Computing},
  volume  = {7},
  number  = {4},
  pages   = {1160--1174},
  year    = {1986},
  doi     = {10.1137/0907079}
}

@article{wen2013feasible,
  author  = {Wen, Zaiwen and Yin, Wotao},
  title   = {A feasible method for optimization with orthogonality constraints},
  journal = {Mathematical Programming},
  volume  = {142},
  number  = {1--2},
  pages   = {397--434},
  year    = {2013},
  doi     = {10.1007/s10107-012-0584-1}
}

@article{gao2018framework,
  author  = {Gao, Bin and Liu, Xin and Chen, Xiaojun and Yuan, Ya-xiang},
  title   = {A New First-Order Algorithmic Framework for Optimization Problems with Orthogonality Constraints},
  journal = {SIAM Journal on Optimization},
  volume  = {28},
  number  = {1},
  pages   = {302--332},
  year    = {2018},
  doi     = {10.1137/16M1098759}
}

@inproceedings{li2020cayley,
  author    = {Li, Jun and Li, Fuxin and Todorovic, Sinisa},
  title     = {Efficient Riemannian Optimization on the Stiefel Manifold via the Cayley Transform},
  booktitle = {International Conference on Learning Representations (ICLR)},
  year      = {2020}
}

@inproceedings{bernstein2018signsgd,
  author    = {Bernstein, Jeremy and Wang, Yu-Xiang and Azizzadenesheli, Kamyar and Anandkumar, Anima},
  title     = {signSGD: Compressed Optimisation for Non-Convex Problems},
  booktitle = {Proceedings of the 35th International Conference on Machine Learning (ICML)},
  series    = {Proceedings of Machine Learning Research},
  volume    = {80},
  pages     = {560--569},
  year      = {2018}
}

@inproceedings{bernstein2019signsgdmajority,
  author    = {Bernstein, Jeremy and Zhao, Jiawei and Azizzadenesheli, Kamyar and Anandkumar, Anima},
  title     = {signSGD with Majority Vote is Communication Efficient and Fault Tolerant},
  booktitle = {7th International Conference on Learning Representations (ICLR)},
  publisher = {OpenReview.net},
  year      = {2019}
}

@inproceedings{alistarh2017qsgd,
  author    = {Alistarh, Dan and Grubic, Demjan and Li, Jerry and Tomioka, Ryota and Vojnovic, Milan},
  title     = {{QSGD:} Communication-Efficient {SGD} via Gradient Quantization and Encoding},
  booktitle = {Advances in Neural Information Processing Systems 30 (NeurIPS 2017)},
  pages     = {1709--1720},
  year      = {2017}
}

@inproceedings{gunasekar2018characterizing,
  title={Characterizing Implicit Bias in Terms of Optimization Geometry},
  author={Gunasekar, Suriya and Lee, Jason and Soudry, Daniel and Srebro, Nathan},
  booktitle={Proceedings of the 35th International Conference on Machine Learning},
  year={2018},
  publisher={PMLR},
  volume={80},
  pages={1855--1864}
}

@article{bernstein2024old,
  title={Old Optimizer, New Norm: An Anthology},
  author={Bernstein, Jeremy and Newhouse, Laker},
  journal={arXiv preprint arXiv:2409.03137},
  year={2024}
}

@inproceedings{lyu2020gradient,
  title={Gradient Descent Maximizes the Margin of Homogeneous Neural Networks},
  author={Lyu, Kaifeng and Li, Jian},
  booktitle={International Conference on Learning Representations},
  year={2020}
}

@article{soudry2018implicit,
  title={The Implicit Bias of Gradient Descent on Separable Data},
  author={Soudry, Daniel and Hoffer, Elad and Nacson, Mor Shpigel and Gunasekar, Suriya and Srebro, Nathan},
  journal={The Journal of Machine Learning Research},
  volume={19},
  number={1},
  pages={2822--2878},
  year={2018},
  publisher={JMLR. org}
}

@inproceedings{vasudeva2024rich,
  title={The Rich and the Simple: On the Implicit Bias of {Adam} and {SGD}},
  author={Vasudeva, Bhavya and De Bortoli, Valentin andr Courville, Aaron and Mitliagkas, Ioannis},
  booktitle={Advances in Neural Information Processing Systems},
  volume={37},
  year={2024}}

@inproceedings{wang2021momentum,
  title={Does Momentum Change the Implicit Regularization on Separable Data?},
  author={Wang, Bohan and Meng, Qi and Zhang, Huishuai and Sun, Ruoyu and Chen, Wei and Ma, Zhi-Ming and Liu, Tie-Yan},
  booktitle={Advances in Neural Information Processing Systems},
  volume={34},
  pages={25384--25397},
  year={2021}
}

@inproceedings{wang2021implicit,
  title={The Implicit Bias for Adaptive Optimization Algorithms on Homogeneous Neural Networks},
  author={Wang, Bohan and Meng, Qi and Chen, Wei and Liu, Tie-Yan},
  booktitle={Proceedings of the 38th International Conference on Machine Learning},
  pages={10849--10858},
  year={2021},
  organization={PMLR}}

@article{li2023implicit,
  title={Implicit Bias of Deep Learning in the Large Learning Rate Phase: A Data Separability Perspective},
  author={Li, Mulang and Liu, Shuo and Wang, Irwin King},
  journal={Applied Sciences},
  volume={13},
  number={6},
  pages={3961},
  year={2023},
  publisher={MDPI}
}

@inproceedings{gupta2018shampoo,
  title={Shampoo: Preconditioned stochastic tensor optimization},
  author={Gupta, Vineet and Koren, Tomer and Singer, Yoram},
  booktitle={International Conference on Machine Learning},
  pages={1842--1850},
  year={2018},
  organization={PMLR}
}

@inproceedings{karimireddy2019errorfeedback,
  author    = {Karimireddy, Sai Praneeth and Rebjock, Quentin and Stich, Sebastian and Jaggi, Martin},
  title     = {Error Feedback Fixes SignSGD and other Gradient Compression Schemes},
  booktitle = {Proceedings of the 36th International Conference on Machine Learning (ICML)},
  series    = {Proceedings of Machine Learning Research},
  volume    = {97},
  pages     = {3252--3261},
  year      = {2019}
}

@article{page2025muonall,
  title={MuonAll: Muon Variant for Efficient Finetuning of Large Language Models},
  author={Page, Saurabh and Joshi, Advait and Sonawane, S. S.},
  journal={arXiv preprint arXiv:2511.06086},
  year={2025}
}

@article{jordan2024muon,
  title={Muon: An optimizer for hidden layers in neural networks},
  author={Jordan, Keller and Jin, Yuchen and Boza, Vlado and You, Jiacheng and Cesista, Franz and Newhouse, Laker and Bernstein, Jeremy},
  journal={Blog post},
  url={https://kellerjordan.github.io/posts/muon/},
  year={2024}
}

@article{chang2025convergence,
  title={On the Convergence of Muon and Beyond},
  author={Chang, Da and Liu, Yongxiang and Yuan, Ganzhao},
  journal={arXiv preprint},
  year={2025}
}

@inproceedings{goldstein2025exploration,
  title={An Exploration of Non-Euclidean Gradient Descent: Muon and its Many Variants},
  author={Goldstein, Samuel and Grangier, David and Menzies, James},
  booktitle={OpenReview},
  year={2025}
}

@article{li2025normuon,
  title={NorMuon: Making Muon more efficient and scalable},
  author={Li, Zichong and Liu, Liming and Liang, Chen and Chen, Weizhu and Zhao, Tuo},
  journal={arXiv preprint},
  year={2025}
}

@article{boissin2025turbo,
  title={Turbo-Muon: Accelerating Newton-Schulz for Large Scale Optimization},
  author={Boissin, Thibaut and Massena, Thomas and Mamalet, Franck and Serrurier, Mathieu},
  journal={arXiv preprint},
  year={2025}
}

@article{gupta2025effective,
  title={Effective Quantization of Muon Optimizer States},
  author={Gupta, Aman and Celente, Rafael and Shivanna, Abhishek and Braithwaite, D. T. and Dexter, Gregory and Tang, Shao and Udagawa, Hiroto and Silva, Daniel and Ramanath, Rohan and Keerthi, S. Sathiya},
  journal={arXiv preprint arXiv:2509.23106},
  year={2025}
}

@article{liu2025muon,
  title={Muon is Scalable for LLM Training},
  author={Liu, Jingyuan and Su, Jianlin and Yao, Xingcheng and Jiang, Zhejun and Lai, Guokun and Du, Yulun and Qin, Yidao and Xu, Weixin and Lu, Enzhe and Yan, Junjie and Chen, Yanru and Zheng, Huabin and Liu, Yibo and Liu, Shaowei and Yin, Bohong and He, Weiran and Zhu, Han and Wang, Yuzhi and Wang, Jianzhou and Dong, Mengnan and Zhang, Zheng and Kang, Yongsheng and Zhang, Hao and Xu, Xinran and Zhang, Yutao and Wu, Yuxin and Zhou, Xinyu and Yang, Zhilin},
  journal={arXiv preprint arXiv:2502.16982},
  year={2025}
}

@article{si2025adamuon,
  title={AdaMuon: Adaptive Muon Optimizer},
  author={Si, Chongjie and Zhang, Debing and Shen, Wei},
  journal={arXiv preprint arXiv:2507.11005},
  year={2025}
}

@inproceedings{chen2023symbolic,
  title={Symbolic Discovery of Optimization Algorithms},
  author={Chen, Xiangning and Liang, Chen and Huang, Da and Real, Esteban and Wang, Kaiyu and Liu, Hieu and Pham, Hieu and Dong, Xuanyi and Luong, Thang and Hsieh, Cho-Jui and Lu, Yifeng and Le, Quoc V.},
  booktitle={Advances in Neural Information Processing Systems},
  volume={36},
  year={2023}
}

@article{sfyraki2025lions,
  title={Lions and Muons: Optimization via Stochastic Frank-Wolfe},
  author={Sfyraki, Amalia and Wang, Yifan and Cevher, Volkan},
  journal={arXiv preprint arXiv:2506.04192},
  year={2025}
}

@article{jiang2025convergence,
  title={Convergence Analysis of the Lion Optimizer in Centralized and Distributed Settings},
  author={Jiang, Wei and Zhang, Lijun},
  journal={arXiv preprint arXiv:2508.12327v1},
  year={2025}
}

@article{rong2025refined,
  title={A Refined Lion Optimizer for Deep Learning},
  author={Rong, Jian and Ma, Chenhao and Zhang, Qinghui and Cao, Yong and Kou, Weili},
  journal={Scientific Reports},
  volume={15},
  number={1},
  pages={1--15},
  year={2025},
  publisher={Nature Publishing Group}
}

@article{duchi11a,
  title   = {Adaptive Subgradient Methods for Online Learning and Stochastic Optimization},
  author  = {Duchi, John and Hazan, Elad and Singer, Yoram},
  journal = {Journal of Machine Learning Research},
  volume  = {12},
  number  = {61},
  pages   = {2121--2159},
  year    = {2011}
}

@misc{xie2024implicit,
  title         = {Implicit Bias of AdamW: {$\ell_\infty$} Norm Constrained Optimization},
  author        = {Xie, Shuo and Li, Zhiyuan},
  year          = {2024},
  eprint        = {2404.04454},
  archivePrefix = {arXiv},
  primaryClass  = {cs.LG}
}

@inproceedings{shazeer2018adafactor,
  title     = {Adafactor: Adaptive Learning Rates with Sublinear Memory Cost},
  author    = {Shazeer, Noam and Stern, Mitchell},
  booktitle = {Proceedings of the 35th International Conference on Machine Learning},
  series    = {Proceedings of Machine Learning Research},
  volume    = {80},
  pages     = {4596--4604},
  year      = {2018},
  editor    = {Dy, Jennifer and Krause, Andreas},
  publisher = {PMLR},
  month     = jul
}

@inproceedings{anil2019memory,
  title     = {Memory Efficient Adaptive Optimization},
  author    = {Anil, Rohan and Gupta, Vineet and Koren, Tomer and Singer, Yoram},
  booktitle = {Advances in Neural Information Processing Systems},
  volume    = {32},
  pages     = {9746--9755},
  year      = {2019}
}

@misc{kingma2014adam,
    author = {Kingma, Diederik P. and Ba, Jimmy},
    title = {Adam: A Method for Stochastic Optimization},
    publisher = {arXiv},
    year = {2014},
}

@inproceedings{loshchilov2019adamw,
    title={Decoupled Weight Decay Regularization},
    author={Ilya Loshchilov and Frank Hutter},
    booktitle={International Conference on Learning Representations},
    year={2019},
}

@inproceedings{luo2023came,
  title     = {{CAME}: Confidence-guided Adaptive Memory Efficient Optimization},
  author    = {Luo, Yang and Ren, Xiaozhe and Zheng, Zangwei and Jiang, Zhuo and Jiang, Xin and You, Yang},
  booktitle = {Proceedings of the 61st Annual Meeting of the Association for Computational Linguistics (Volume 1: Long Papers)},
  pages     = {4442--4453},
  year      = {2023},
  month     = jul,
  address   = {Toronto, Canada},
  publisher = {Association for Computational Linguistics},
  doi       = {10.18653/v1/2023.acl-long.243}
}

@misc{lv2023adalomo,
  title         = {AdaLomo: Low-memory Optimization with Adaptive Learning Rate},
  author        = {Lv, Kai and Yan, Hang and Guo, Qipeng and Lv, Haijun and Qiu, Xipeng},
  year          = {2023},
  eprint        = {2310.10195},
  archivePrefix = {arXiv},
  primaryClass  = {cs.LG}
}

@inproceedings{wang2023closing,
  title={Closing the gap between the upper bound and lower bound of Adam's iteration complexity},
  author={Wang, Bohan and Fu, Jingwen and Zhang, Huishuai and Zheng, Nanning and Chen, Wei},
  booktitle={Thirty-seventh Conference on Neural Information Processing Systems},
  year={2023}
}

@inproceedings{wang2024provable,
  title={Provable adaptivity of adam under non-uniform smoothness},
  author={Wang, Bohan and Zhang, Yushun and Zhang, Huishuai and Meng, Qi and Sun, Ruoyu and Ma, Zhi-Ming and Liu, Tie-Yan and Luo, Zhi-Quan and Chen, Wei},
  booktitle={Proceedings of the 30th ACM SIGKDD Conference on Knowledge Discovery and Data Mining},
  pages={2960--2969},
  year={2024}
}

@article{li2023convergence,
  title={Convergence of adam under relaxed assumptions},
  author={Li, Haochuan and Rakhlin, Alexander and Jadbabaie, Ali},
  journal={Advances in Neural Information Processing Systems},
  volume={36},
  pages={52166--52196},
  year={2023}
}

@inproceedings{zhang2025adams,
  title     = {{A}dam{S}: Momentum Itself Can Be A Normalizer for {LLM} Pretraining and Post-training},
  author    = {Zhang, Huishuai and Wang, Bohan and Chen, Luoxin},
  booktitle = {Proceedings of the 2025 Conference on Empirical Methods in Natural Language Processing},
  pages     = {10719--10738},
  year      = {2025},
  month     = nov,
  address   = {Suzhou, China},
  publisher = {Association for Computational Linguistics},
  doi       = {10.18653/v1/2025.emnlp-main.543}
}

@misc{modded_nanogpt_2024,
  author       = {Keller Jordan and Jeremy Bernstein and Brendan Rappazzo and
                  @fernbear.bsky.social and Boza Vlado and You Jiacheng and
                  Franz Cesista and Braden Koszarsky and @Grad62304977},
  title        = {modded-nanogpt: Speedrunning the NanoGPT baseline},
  year         = {2024},
  url          = {https://github.com/KellerJordan/modded-nanogpt}
}

@article{zhang2024adam,
  title     = {Adam-mini: Use Fewer Learning Rates To Gain More},
  author    = {Zhang, Yushun and Chen, Congliang  and Li, Ziniu and Ding, Tian and Wu, Chenwei and Kingma, Diederik P and Ye, Yinyu and Luo, Zhi-Quan and Sun, Ruoyu},
  booktitle = {arXiv preprint arXiv:2406.16793},
  year      = {2024},
}

@inproceedings{
   liang2025torchtitan,
   title={TorchTitan: One-stop PyTorch native solution for production ready {LLM} pretraining},
   author={Wanchao Liang and Tianyu Liu and Less Wright and Will Constable and Andrew Gu and Chien-Chin Huang and Iris Zhang and Wei Feng and Howard Huang and Junjie Wang and Sanket Purandare and Gokul Nadathur and Stratos Idreos},
   booktitle={The Thirteenth International Conference on Learning Representations},
   year={2025},
   url={https://openreview.net/forum?id=SFN6Wm7YBI}
}

@article{imagenet15russakovsky,
    Author = {Olga Russakovsky and Jia Deng and Hao Su and Jonathan Krause and Sanjeev Satheesh and Sean Ma and Zhiheng Huang and Andrej Karpathy and Aditya Khosla and Michael Bernstein and Alexander C. Berg and Li Fei-Fei},
    Title = { {ImageNet Large Scale Visual Recognition Challenge} },
    Year = {2015},
    journal   = {International Journal of Computer Vision (IJCV)},
    doi = {10.1007/s11263-015-0816-y},
    volume={115},
    number={3},
    pages={211-252}
}

@misc{liu2025muonscalablellmtraining,
      title={Muon is Scalable for LLM Training}, 
      author={Jingyuan Liu and Jianlin Su and Xingcheng Yao and Zhejun Jiang and Guokun Lai and Yulun Du and Yidao Qin and Weixin Xu and Enzhe Lu and Junjie Yan and Yanru Chen and Huabin Zheng and Yibo Liu and Shaowei Liu and Bohong Yin and Weiran He and Han Zhu and Yuzhi Wang and Jianzhou Wang and Mengnan Dong and Zheng Zhang and Yongsheng Kang and Hao Zhang and Xinran Xu and Yutao Zhang and Yuxin Wu and Xinyu Zhou and Zhilin Yang},
      year={2025},
      eprint={2502.16982},
      archivePrefix={arXiv},
      primaryClass={cs.LG},
      url={https://arxiv.org/abs/2502.16982}, 
}

@misc{zhang2024implicitbiasadamseparable,
      title={The Implicit Bias of Adam on Separable Data}, 
      author={Chenyang Zhang and Difan Zou and Yuan Cao},
      year={2024},
      eprint={2406.10650},
      archivePrefix={arXiv},
      primaryClass={stat.ML},
      url={https://arxiv.org/abs/2406.10650}, 
}

@misc{chen2025lionsecretlysolvesconstrained,
      title={Lion Secretly Solves Constrained Optimization: As Lyapunov Predicts}, 
      author={Lizhang Chen and Bo Liu and Kaizhao Liang and Qiang Liu},
      year={2025},
      eprint={2310.05898},
      archivePrefix={arXiv},
      primaryClass={cs.LG},
      url={https://arxiv.org/abs/2310.05898}, 
}

\appendix
\newpage
\appendix
\onecolumn
\section{Experimental Details}
\label{app:exp_details}

\paragraph{Software and data availability.}
All experiments in this paper are conducted using widely adopted open-source training repositories and standard public datasets with strong community support. To facilitate reproducibility, we next provide detailed descriptions of the training setups and hyperparameter configurations used in each experiment. Our \nameA{} implementation (including scripts and configurations needed to reproduce the reported results) is available at \codeurl.


\subsection{Hardware and Software Environment}

All experiments are conducted with Python 3.10+ and PyTorch 2.0+ (CUDA 11.8+). For GPT-2 experiments, we use 4 NVIDIA A100 80GB GPUs with DistributedDataParallel (DDP). For Llama2-7B pretraining and Llama-3.1-8B supervised fine-tuning, we use 8 NVIDIA RTX PRO 6000 GPUs with FSDP2 (Fully Sharded Data Parallel) for distributed training. SiT-B/2 pretraining uses 4 NVIDIA RTX PRO GPUs. All experiments use mixed precision training with bfloat16 when supported; otherwise we use float32. We enable PyTorch 2.0's \texttt{torch.compile} to improve training efficiency.

\subsection{GPT-2 Pretraining Experiments}
\label{app:gpt2_details}

We start from the modded-nanogpt codebase by Keller Jordan~\cite{modded_nanogpt_2024} and train GPT-2 models of four sizes: Small (124M parameters), Medium (355M parameters), Large (770M parameters), and XL (1.5B parameters). The Small model has $n_{\text{layer}}=12$, $n_{\text{head}}=12$, and $n_{\text{embd}}=768$. The Medium model has $n_{\text{layer}}=24$, $n_{\text{head}}=16$, and $n_{\text{embd}}=1024$. The Large model has $n_{\text{layer}}=36$, $n_{\text{head}}=20$, and $n_{\text{embd}}=1280$. The XL model has $n_{\text{layer}}=48$, $n_{\text{head}}=25$, and $n_{\text{embd}}=1600$. We use the standard GPT-2 architecture without bias terms in LayerNorm and Linear layers, and dropout is disabled for pretraining.

We train on the OpenWebText dataset, tokenized with the standard GPT-2 tokenizer. The context length is set to 1024 tokens. We use a micro-batch size of 12 per GPU with gradient accumulation over 40 steps. We train the Small, Medium, and Large models for 20,000 steps, while the XL model is trained for 10,000 steps. The maximum learning rate is set to $6 \times 10^{-4}$ with a linear warmup for 2,000 steps, followed by a cosine decay schedule that reduces the learning rate to $6 \times 10^{-5}$. Weight decay is set to $0.1$ for all optimizers, and gradient clipping is applied at 1.0. We evaluate every 2,000 steps using 200 evaluation iterations.

For all optimizers, we use the first-order momentum $\beta_1=0.9$ to align with the Adam baseline for fair comparison. AdamW uses $\beta_2=0.95$ and $\epsilon=10^{-8}$. Muon uses a single momentum coefficient $\beta=0.95$ and performs Newton-Schulz orthogonalization with $K=5$ iterations. OLion uses $\beta_1=0.95$ and $\beta_2=0.98$ for its double-momentum scheme, with Newton-Schulz steps $K=5$. AdaMuon uses $\beta=0.95$ for momentum, and $K=5$ Newton-Schulz steps.

\subsection{Llama-7B Pretraining}
\label{app:llama7b_details}

We modified on the codebase of Adam-mini~\cite{zhang2024adam} and train the Llama-2-7B architecture (7 billion parameters) using the standard Llama-2 configuration with RMSNorm, SwiGLU activation, and rotary positional embeddings (RoPE). We use FSDP2 (Fully Sharded Data Parallel) for distributed training across multiple nodes. The effective global batch size is set to 4M tokens. Specifically,  with a context length of 4096 tokens, a local batch size of 2, and gradient accumulation step of 64. We train for a total of 8192 steps. The maximum learning rate is $3 \times 10^{-4}$ with a linear warmup for 2,000 steps, followed by a linear decay schedule that gradually reduce the learning rate to $3 \times 10^{-5}$. Weight decay is set to $0.1$ and gradient clipping is applied at 1.0. All training is conducted in bfloat16 mixed precision. Validation loss is computed every 100 intervals during training.

We use large-scale text datasets suitable for billion-parameter models, following the standard Llama-2 data preprocessing pipeline. Tokenization is performed using the SentencePiece tokenizer, consistent with the original Llama-2 setup. For optimizer hyperparameters, AdamW uses $\beta_1=0.9$, $\beta_2=0.95$, and $\epsilon=10^{-8}$. Muon and AdaMuon use $\beta=0.95$ with Newton-Schulz steps $K=5$. OLion uses $\beta_1=0.95$, $\beta_2=0.98$, and Newton-Schulz steps $K=5$. 

\subsection{SiT-B/2 Pretraining}
\label{app:sit_details}

We train SiT-B/2 (Scalable Transformer for Diffusion Models) using the standard SiT architecture configuration. The training dataset is ImageNet-1K~\cite{imagenet15russakovsky} with images at $256 \times 256$ resolution, using the standard ImageNet augmentation pipeline. We use a learning rate of $1 \times 10^{-4}$ and train for 400,000 steps with a cosine decay learning rate schedule. The batch size is adjusted per GPU memory constraints. Weight decay is set to $0.1$ and gradient clipping is applied at 1.0. Training is conducted in bfloat16 or float16 mixed precision. We track both projection loss and denoising loss separately, using the standard diffusion projection loss and denoising loss formulations. For optimizer hyperparameters, Muon uses $\beta=0.95$ with Newton-Schulz steps $K=5$, and OLion uses $\beta_1=0.95$, $\beta_2=0.98$, and Newton-Schulz steps $K=5$.

\section{Diagonal-isotropy Verification for a Random Gaussian Model}
\label{app:random-diagonal-isotropy}

In this section we introduce a standard random model for the singular vectors $(\mathbf U,\mathbf V)$, and show that the diagonal-isotropy condition in Assumption~\ref{assm:diag-balance} holds with high probability with an explicit parameter~$\varepsilon$.

\subsection{Random model}
\label{sec:random-model}

Let $r\ge 2$ and $d_1,d_2\ge r$. We generate $(\mathbf U,\mathbf V)$ as follows:
\begin{enumerate}
\item Sample $\mathbf A\in\mathbb R^{d_1\times r}$ and $\mathbf B\in\mathbb R^{d_2\times r}$ with i.i.d.\ $\mathcal N(0,1)$ entries, independently.
\item Let $\mathbf U=\mathrm{qf}(\mathbf A)\in\mathbb R^{d_1\times r}$ and $\mathbf V=\mathrm{qf}(\mathbf B)\in\mathbb R^{d_2\times r}$ be the orthonormal $Q$-factors in the QR decompositions of $\mathbf A$ and $\mathbf B$.
\end{enumerate}
Then $\mathbf U$ and $\mathbf V$ are independent Haar-distributed random matrices on the Stiefel manifolds (i.e., their column spaces are uniformly random), and satisfy $\mathbf U^\top\mathbf U=\mathbf I_r$, $\mathbf V^\top\mathbf V=\mathbf I_r$.

Define
\[
\mathbf Q := \mathbf U\mathbf V^\top\in\mathbb R^{d_1\times d_2},
\qquad
\mathbf Z := \sign(\mathbf Q) \ \ \text{(entrywise, with $\sign(0)=0$)}.
\]
We also define the diagonal correlation vector
\[
\mathbf m := \diag(\mathbf U^\top \mathbf Z \mathbf V)\in\mathbb R^r,
\qquad
\bar m := \frac{1}{r}\mathbf 1^\top \mathbf m.
\]
A key identity is
\begin{equation}
\label{eq:trace-identity}
\mathbf 1^\top \mathbf m
=
\tr(\mathbf U^\top \mathbf Z \mathbf V)
=
\langle \mathbf U\mathbf V^\top,\mathbf Z\rangle
=
\langle \mathbf Q,\sign(\mathbf Q)\rangle
=
\|\mathbf Q\|_1,
\qquad
\Rightarrow\quad
\bar m=\frac{\|\mathbf Q\|_1}{r}.
\end{equation}

\subsection{Main claim: diagonal-isotropy holds with high probability}
\label{sec:main-claim}

\begin{theorem}[Diagonal-isotropy for random Gaussian singular vectors]
\label{thm:diag-isotropy-random}
Let $(\mathbf U,\mathbf V)$ follow the random model in Section~\ref{sec:random-model}, and define
$\mathbf Q=\mathbf U\mathbf V^\top$ and $\mathbf Z=\sign(\mathbf Q)$.
Then there exist absolute constants $c,C>0$ such that, with probability at least
\[
1-2r^{-10}-2\exp(-c\,d_1d_2),
\]
the following two events hold simultaneously:
\begin{align}
\label{eq:m-fluct}
\Big\|
\diag(\mathbf U^\top \mathbf Z \mathbf V) - \tfrac{\|\mathbf Q\|_1}{r}\mathbf 1
\Big\|_2
&\le
C\sqrt{r\log r},\\[3pt]
\label{eq:Q1-lower}
\|\mathbf Q\|_1
&\ge
c\sqrt{r\,d_1d_2}.
\end{align}
Consequently, on the same event,
Assumption~\ref{assm:diag-balance} holds with
\begin{equation}
\label{eq:epsilon-scaling}
\varepsilon
=
\frac{C}{c}\sqrt{\frac{r\log r}{d_1d_2}}.
\end{equation}
In particular, if $d_1d_2\gg r\log r$, then $\varepsilon=o(1)$.
\end{theorem}

\subsection{Proof of Theorem~\ref{thm:diag-isotropy-random}}
\label{sec:proof-diag-isotropy}

\paragraph{Step 1: Centering factor.}
By Equation~(\ref{eq:trace-identity}), we have the exact relation
\[
\bar m=\frac{1}{r}\mathbf 1^\top \diag(\mathbf U^\top \mathbf Z \mathbf V)=\frac{\|\mathbf Q\|_1}{r},
\]
which matches the centering used in Assumption~\ref{assm:diag-balance}.
Moreover, by Haar symmetry, the coordinates of $\mathbf m$ are exchangeable (permuting columns of $\mathbf U$ and $\mathbf V$ does not change the joint law), so it is natural that each coordinate concentrates at the same scale.

\paragraph{Step 2: A lower bound on $\|\mathbf Q\|_1$.}
Under independent Haar $\mathbf U,\mathbf V$, each entry
\[
Q_{ij}=\sum_{k=1}^r U_{ik}V_{jk}
\]
is mean-zero and, in the high-dimensional regime, behaves approximately as a Gaussian with variance $\Var(Q_{ij})\approx r/(d_1d_2)$ (a standard consequence of rotational invariance and the fact that each $U_{ik},V_{jk}$ is of size $\Theta(d_1^{-1/2}),\Theta(d_2^{-1/2})$).
Thus $\mathbb E|Q_{ij}|\approx \sqrt{\tfrac{2}{\pi}}\sqrt{\tfrac{r}{d_1d_2}}$ and
\[
\mathbb E\|\mathbf Q\|_1=\sum_{i,j}\mathbb E|Q_{ij}|
\approx \sqrt{\tfrac{2}{\pi}}\sqrt{r\,d_1d_2}.
\]
Furthermore, $\|\mathbf Q\|_1$ concentrates around its mean (as a Lipschitz function of the underlying Gaussian matrices $(\mathbf A,\mathbf B)$ away from negligible ill-conditioning events), yielding the high-probability bound Equation~(\ref{eq:Q1-lower}) for an absolute constant $c>0$.

\paragraph{Step 3: Concentration of diagonal correlations.}
For each $k\in[r]$,
\[
m_k=(\mathbf U^\top \mathbf Z \mathbf V)_{kk}
=
u_k^\top \mathbf Z v_k
=
\sum_{i=1}^{d_1}\sum_{j=1}^{d_2} U_{ik}V_{jk}\,\sign(Q_{ij}).
\]
Heuristically, $(Q_{ij},U_{ik}V_{jk})$ is close to a jointly Gaussian pair with
\[
\Var(Q_{ij})\approx \frac{r}{d_1d_2},
\qquad
\Var(U_{ik}V_{jk})\approx \frac{1}{d_1d_2},
\qquad
\Cov(Q_{ij},U_{ik}V_{jk})\approx \frac{1}{d_1d_2}.
\]
In particular, each summand has variance of order $1/(d_1d_2)$ and the sum over $d_1d_2$ terms yields constant-scale fluctuations for $m_k$. Formally, one may replace $\sign(\cdot)$ by a smooth approximation (e.g.\ $\tanh(\cdot/\tau)$), apply Gaussian concentration to obtain sub-Gaussian tails for the smoothed $m_k$, and then remove the smoothing using anti-concentration of $Q_{ij}$ near zero. This yields an absolute constant $c_0>0$ such that
\[
\mathbb P\!\left(|m_k-\mathbb E m_k|\ge t\right)\le 2\exp(-c_0 t^2),
\qquad \forall t\ge 0.
\]
Taking a union bound over $k=1,\dots,r$ gives, with probability at least $1-2r^{-10}$,
\[
\max_{k\in[r]} |m_k-\mathbb E m_k|
\ \le\
C_0\sqrt{\log r}
\]
for an absolute constant $C_0>0$, hence
\[
\|\mathbf m-\mathbb E\mathbf m\|_2
\le
\sqrt r \max_{k\in[r]}|m_k-\mathbb E m_k|
\le
C_0\sqrt{r\log r}.
\]
Finally, since $\bar m=\|\mathbf Q\|_1/r$ is the empirical mean of $(m_k)$ by Equation~(\ref{eq:trace-identity}) and $\|\mathbf Q\|_1$ concentrates (Step~2), the deviation between centering by $\mathbb E m_k$ and by $\bar m$ is of the same (or smaller) order. Absorbing constants yields Equation~(\ref{eq:m-fluct}) with an absolute constant $C$.

\paragraph{Step 4: Conclude Assumption~\ref{assm:diag-balance}.}
On the event where Equation~(\ref{eq:m-fluct}) and Equation~(\ref{eq:Q1-lower}) both hold,
\[
\Big\|
\diag(\mathbf U^\top \sign(\mathbf U\mathbf V^\top)\mathbf V)-\frac{\|\mathbf U\mathbf V^\top\|_1}{r}\mathbf 1
\Big\|_2
\le
C\sqrt{r\log r}
=
\frac{C}{c}\sqrt{\frac{r\log r}{d_1d_2}}\cdot \frac{\|\mathbf Q\|_1}{\sqrt r},
\]
which is exactly Assumption~\ref{assm:diag-balance} with $\varepsilon$ as in Equation~(\ref{eq:epsilon-scaling}).
This completes the proof.

\section{Diagonal-Isotropy Verification Along the Training Trajectory}
\label{app:diagonal-isotropy-verification}
Empirically, we verify that the matrix updates along the training trajectory satisfy Assumption~\ref{assm:diag-balance} with small values of $\varepsilon$. Concretely, during GPT-2 small (124M) pretraining we record the value of $\varepsilon$ (i.e., the smallest constant such that the diagonal-isotropy inequality holds) for the gradient-related matrix $\widetilde{\mathbf{G}}_t$ at each update, under three optimizers: AdamW, Muon, and OLion. The results are shown in Figure~\ref{fig:diag-isotropy-1} and Figure~\ref{fig:diag-isotropy-2}.

\begin{figure}[htbp]
  \centering
  \begin{subfigure}[t]{0.48\linewidth}
    \centering
    \includegraphics[width=\linewidth]{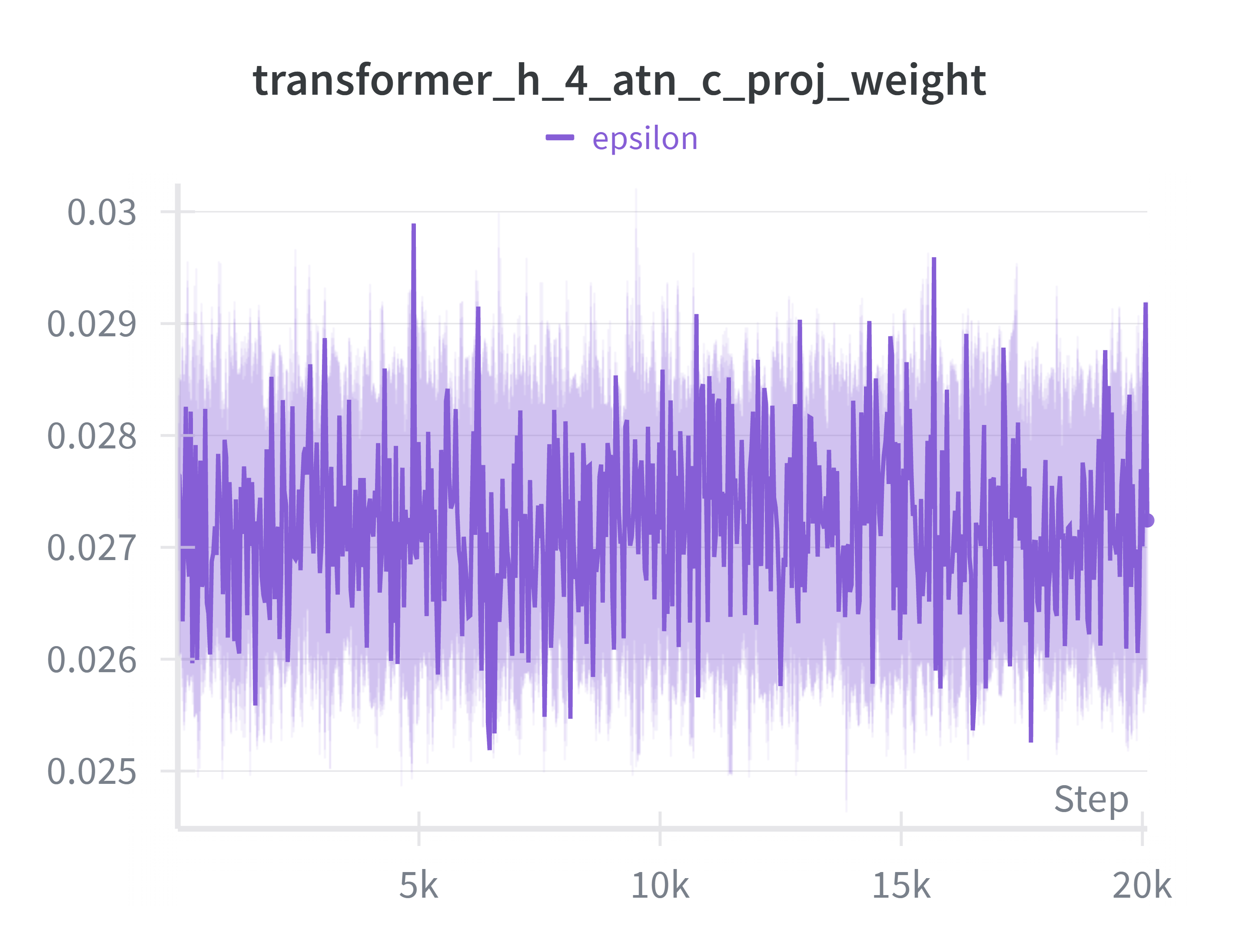}
    \caption{Attention matrix (layer 4), shape $768\times 768$.}
    \label{fig:diag-isotropy-1}
  \end{subfigure}
  \hfill
  \begin{subfigure}[t]{0.48\linewidth}
    \centering
    \includegraphics[width=\linewidth]{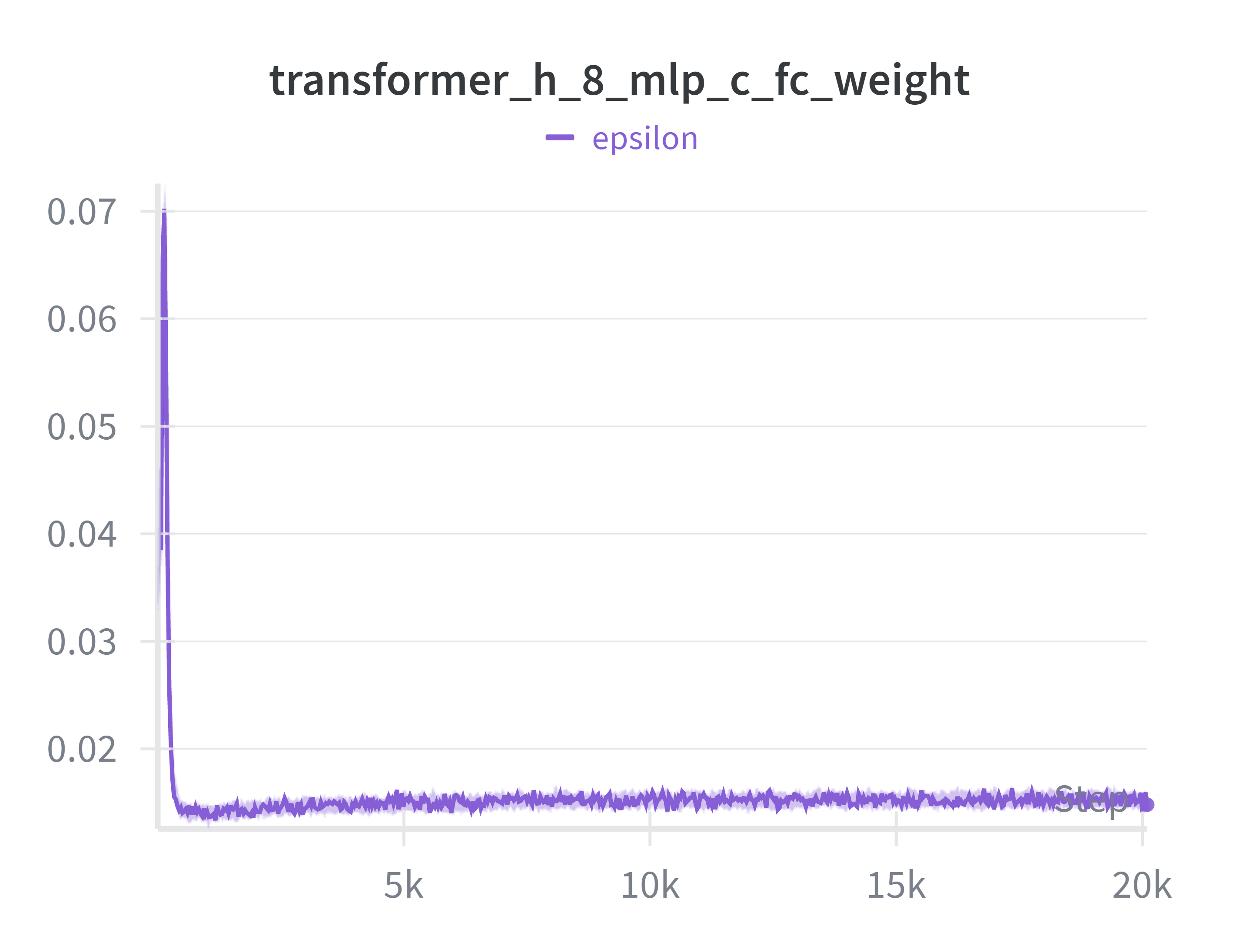}
    \caption{MLP projection matrix (layer 8), shape $3,072\times 768$.}
    \label{fig:diag-isotropy-2}
  \end{subfigure}
  \caption{The $\varepsilon$ values are consistently small for different matrices across the training procedure.}
  \label{fig:diag-isotropy}
\end{figure}


\section{Proofs of Lemma~\ref{lem:upper-bound}} \label{app:proof-lemma-upperbound}

\begin{proof}
Let $\mQ := \mU\mV^\top$ and $\mS := \sign(\mQ)$ (entrywise, with $\sign(0)=0$). Since
$\mZ=\mU\boldsymbol{\Sigma}\mV^\top$ and $\alpha=\tr(\boldsymbol{\Sigma})/r$, we have
\[
\mZ-\alpha \mU\mV^\top = \mU(\boldsymbol{\Sigma}-\alpha \mI)\mV^\top.
\]
Using the Frobenius inner product $\langle \mA,\mB\rangle=\tr(\mA^\top \mB)$ and the orthonormality
$\mU^\top \mU=\mI_r$, $\mV^\top \mV=\mI_r$, we obtain
\begin{align*}
\big\langle \mZ-\alpha \mU\mV^\top,\ \mS\big\rangle
= \big\langle \mU(\boldsymbol{\Sigma}-\alpha \mI)\mV^\top,\ \mS\big\rangle = \tr\!\Big((\boldsymbol{\Sigma}-\alpha \mI)\,\mU^\top \mS\,\mV\Big).
\end{align*}
Since $\boldsymbol{\Sigma}-\alpha \mI$ is diagonal, only the diagonal entries of $\mU^\top \mS\mV$
contribute to the trace. Define
\[
\mathbf{m}:=\diag(\mU^\top \mS\mV)\in\mathbb{R}^r,
\qquad
\beta:=\frac{1}{r}\mathbf{1}^\top \mathbf{m}.
\]
Then
\[
\tr\!\Big((\boldsymbol{\Sigma}-\alpha \mI)\,\mU^\top \mS\,\mV\Big)
=\sum_{k=1}^r (\sigma_k-\alpha)\,m_k
= \big\langle \boldsymbol{\sigma}-\alpha\mathbf{1},\ \mathbf{m}\big\rangle,
\]
where $\boldsymbol{\sigma}=(\sigma_1,\ldots,\sigma_r)^\top$ and $m_k$ is the $k$-th entry of $\mathbf{m}$.

We next relate $\beta$ to $\|\mQ\|_1$. By cyclicity of trace,
\begin{flalign*}
\mathbf{1}^\top \mathbf{m}=\tr(\mU^\top \mS\mV)
=\tr(\mV\mU^\top \mS)=\langle \mU\mV^\top,\ \mS\rangle
=\langle \mQ,\ \sign(\mQ)\rangle
=\|\mQ\|_1,
\end{flalign*}
hence $\beta=\|\mU\mV^\top\|_1/r$.

Now decompose $\mathbf{m}=\beta\mathbf{1}+(\mathbf{m}-\beta\mathbf{1})$:
\begin{align*}
\big\langle \boldsymbol{\sigma}-\alpha\mathbf{1},\ \mathbf{m}\big\rangle
&= \beta\,\big\langle \boldsymbol{\sigma}-\alpha\mathbf{1},\ \mathbf{1}\big\rangle
+ \big\langle \boldsymbol{\sigma}-\alpha\mathbf{1},\ \mathbf{m}-\beta\mathbf{1}\big\rangle.
\end{align*}
Because $\alpha=\tr(\boldsymbol{\Sigma})/r$, we have
$\langle \boldsymbol{\sigma}-\alpha\mathbf{1},\ \mathbf{1}\rangle
=\tr(\boldsymbol{\Sigma})-\alpha r=0$, so the first term vanishes. Therefore,
\begin{flalign*}
    \Big|&\big\langle \mZ-\alpha \mU\mV^\top,\ \sign(\mU\mV^\top)\big\rangle\Big|\\
&=
\Big|\big\langle \boldsymbol{\sigma}-\alpha\mathbf{1},\ \mathbf{m}-\beta\mathbf{1}\big\rangle\Big|\\
&\le \|\boldsymbol{\sigma}-\alpha\mathbf{1}\|_2\,\|\mathbf{m}-\beta\mathbf{1}\|_2.
\end{flalign*}
Noting that $\|\boldsymbol{\sigma}-\alpha\mathbf{1}\|_2=\|\boldsymbol{\Sigma}-\alpha \mI\|_F$,
and applying Assumption~\ref{assm:diag-balance} (i.e., $\|\mathbf{m}-\beta\mathbf{1}\|_2\le
\varepsilon\,\|\mU\mV^\top\|_1/\sqrt{r}$), we conclude
\[
\Big|\big\langle \mZ-\alpha \mU\mV^\top,\ \sign(\mU\mV^\top)\big\rangle\Big|
\le
\varepsilon\,\frac{\|\mU\mV^\top\|_1}{\sqrt{r}}\,
\Big\|\boldsymbol{\Sigma}-\alpha \mI\Big\|_F,
\]
as desired.
\end{proof}

\subsection{Discussion on Lemma~\ref{lem:upper-bound}}
\paragraph{Why the cancellation-aware estimate (Lemma~\ref{lem:upper-bound}) matters.}  A naïve bound to control $\ip{\mathbf{Z}-\alpha\mathbf{Q}}{\mathbf{S}}$ is given by Cauchy-Schwarz:
\[
\big|\ip{\mathbf{Z}-\alpha\mathbf{Q}}{\mathbf{S}}\big|
\le \|\mathbf{Z}-\alpha\mathbf{Q}\|_F\,\|\mathbf{S}\|_F,
\]
but this bound ignores two key structural facts: (i) $\alpha=\tr(\boldsymbol{\Sigma})/r$ makes
$\boldsymbol{\Sigma}-\alpha\mathbf{I}$ \emph{trace} 0, so its positive and negative diagonal deviations
must cancel, and (ii) $\mathbf{S}=\sign(\mathbf{Q})$ is not an arbitrary matrix, but is tightly coupled
to $\mathbf{Q}=\Oop(\mathbf{Z})$. We exploit this structure by rewriting
\[
\ip{\mathbf{Z}-\alpha\mathbf{Q}}{\mathbf{S}}
= \tr\!\Big((\boldsymbol{\Sigma}-\alpha\mathbf{I})\,\mathbf{U}^\top\mathbf{S}\mathbf{V}\Big)
= \sum_{j=1}^{r}(\sigma_{j}-\alpha)\,m_{j}
\]
where $\mathbf{m}:=\diag(\mathbf{U}^\top\mathbf{S}\mathbf{V})$ and then centering $\mathbf{m}$ by its mean $\bar m=\| \mathbf{Q}\|_1/r$, which is admissible because
$\sum_j(\sigma_{j}-\alpha)=0$. This yields
\[
\ip{\mathbf{Z}-\alpha\mathbf{Q}}{\mathbf{S}}
= \sum_{j=1}^{r}(\sigma_{j}-\alpha)\,(m_{j}-\bar m),
\]
so the magnitude is governed by \emph{fluctuations} of the diagonal correlations $m_{j}$ around their average,
rather than by the ambient size $\|\mathbf{S}\|_F=\Theta(\sqrt{d_1d_2})$. Under the diagonal-isotropy assumption,
these fluctuations are small, giving the sharper bound
\[
\big|\ip{\mathbf{Z}-\alpha\mathbf{Q}}{\mathbf{S}}\big|
\;\lesssim\;
\|\boldsymbol{\Sigma}-\alpha\mathbf{I}\|_F\cdot \varepsilon\frac{\|\mathbf{Q}\|_1}{\sqrt{r}},
\]
which can be orders of magnitude tighter than Cauchy--Schwarz bound when $\mathbf{Q}$ is dense
(large $\|\mathbf{Q}\|_1$) and the diagonal correlations are nearly uniform.
\section{Proof of Theorem~\ref{thm:descent_stationarity_simplified}} \label{app:proof-theorem-convergence}

\begin{proof}
By $L$-smoothness (Assumption~\ref{ass:smooth}),
\[
f(\mX_{t+1})
\le
f(\mX_t)
+\ip{\mG_t}{\mX_{t+1}-\mX_t}
+\frac{L}{2}\|\mX_{t+1}-\mX_t\|_F^2.
\]
Using $\mX_{t+1}-\mX_t=-\eta_t \mS_t$ gives
\begin{equation}
\label{eq:descent_basic_simplified_pf}
f(\mX_{t+1})
\le
f(\mX_t)
-\eta_t\ip{\mG_t}{\mS_t}
+\frac{L}{2}\eta_t^2\|\mS_t\|_F^2.
\end{equation}
Next, since $\mS_t=\sign(\mQ_t)$ with $\mQ_t\in\R^{d_1\times d_2}$,
\[
\|\mS_t\|_F^2
=\|\sign(\mQ_t)\|_F^2
=d_1d_2.
\]

It remains to lower bound $\ip{\mG_t}{\mS_t}$. Let $\mG_t=\mU_t\boldsymbol{\Sigma}_t\mV_t^\top$ have rank $r_t$,
and let $\mQ_t=\Oop(\mG_t)=\mU_t\mV_t^\top$. Then
\begin{flalign*}
    &\ip{\mG_t}{\mS_t}
=\big\langle \mG_t,\ \sign(\mQ_t)\big\rangle\\
&=\Big(\big\langle \alpha_t \mQ_t,\ \sign(\mQ_t)\big\rangle +
\big\langle \mG_t-\alpha_t \mQ_t,\ \sign(\mQ_t)\big\rangle
\Big).
\end{flalign*}

The first term equals $\alpha_t \|\mQ_t\|_1$ because
$\langle \mQ_t,\sign(\mQ_t)\rangle = \|\mQ_t\|_1$.
For the second term, Lemma~\ref{lem:upper-bound} yields
\begin{flalign*}
    \Big|\big\langle \mG_t-\alpha_t \mQ_t,\ \sign(\mQ_t)\big\rangle\Big|
\le \varepsilon\,\frac{\|\mQ_t\|_1}{\sqrt{r_t}}\,  \big\|\boldsymbol{\Sigma}_t-\alpha_t \mI\big\|_F =\varepsilon\,\|\mQ_t\|_1\,\alpha_t\,\rho_t.
\end{flalign*}

Therefore,
\[
\ip{\mG_t}{\mS_t}
\ge \|\mQ_t\|_1\,\alpha_t(1-\varepsilon\rho_t)
= \Phi_t.
\]
Plugging this and $\|\mS_t\|_F^2=d_1d_2$ into Equation~(\ref{eq:descent_basic_simplified_pf}) gives
Equation~(\ref{eq:descent_phi_simplified}). Summing Equation~(\ref{eq:descent_phi_simplified}) from $t=0$ to $T-1$ and using $f(\mX_T)\ge f_{\inf}$ yields Equation~(\ref{eq:sum_phi_simplified}).
\end{proof}

\section{Implicit Bias Verification}\label{app:matrix_cases}

We sample several individual matrices from different layers and examine their singular-value and absolute-value
distributions in detail.
Figures~\ref{fig:matrix_h7_mlp}, \ref{fig:matrix_h8_mlp}, \ref{fig:matrix_h10_mlp}, and \ref{fig:matrix_h2_attn}
show representative MLP projection matrices (layers 7, 8, and 10) and an attention projection matrix (layer 2),
respectively.
Consistent with the aggregate statistics, \nameA{} produces more concentrated singular-value distributions (with lower
leading components) than AdamW, while also producing tighter absolute-value distributions (with smaller element
magnitudes) than Muon. This confirms that the combined spectral and $\ell_\infty$ biases manifest robustly across
individual layers and matrix shapes.
\begin{figure}[htbp]
    \centering
    \includegraphics[width=0.95\linewidth]{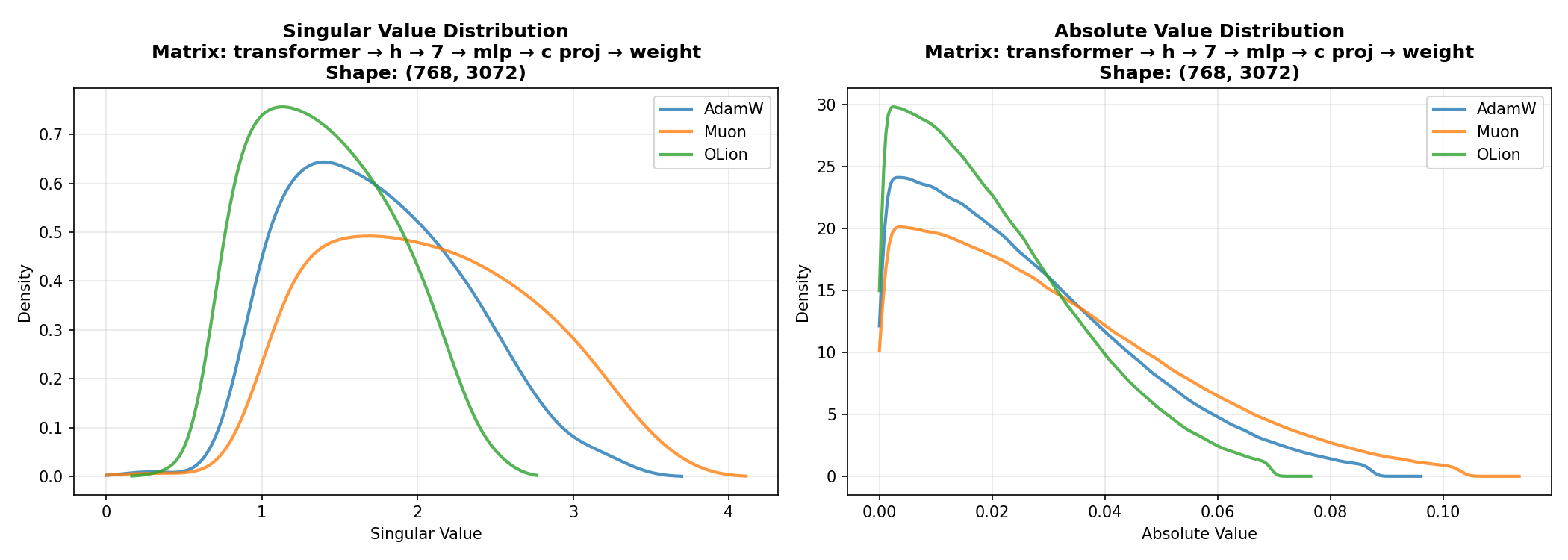}
    \caption{Singular value and absolute value distributions for the MLP projection weight matrix at layer 7 (shape $768\times 3072$).}
    \label{fig:matrix_h7_mlp}
\end{figure}

\begin{figure}[htbp]
    \centering
    \includegraphics[width=0.95\linewidth]{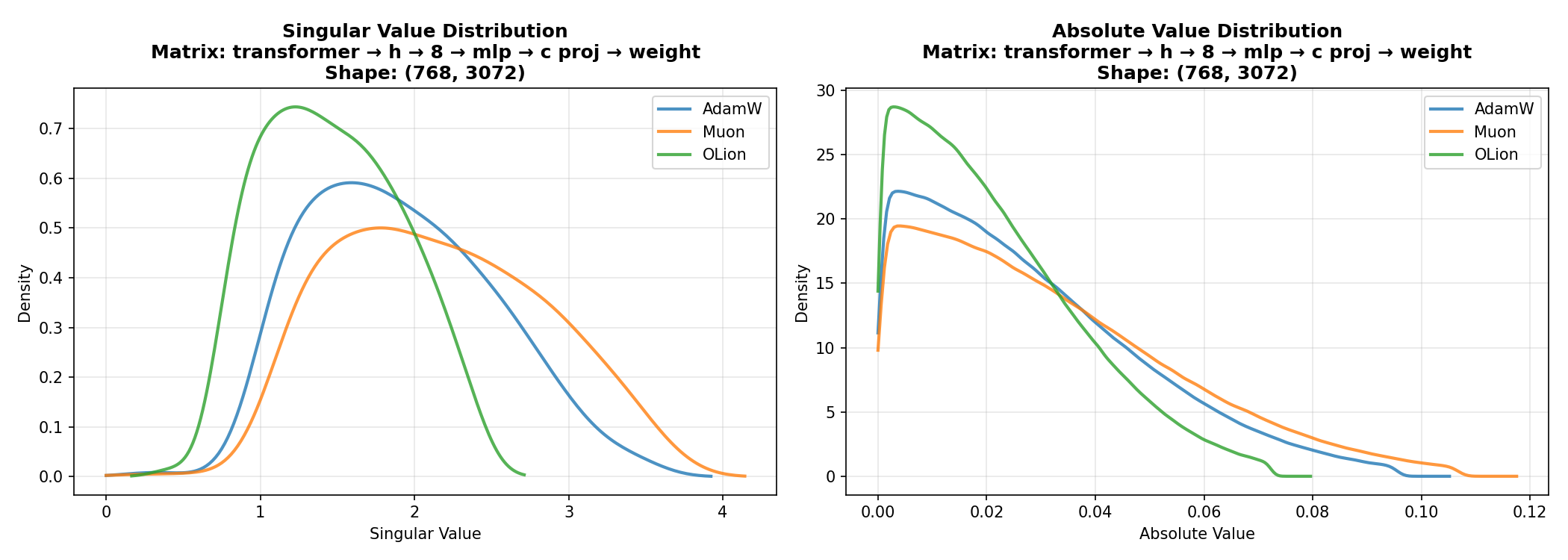}
    \caption{Singular value and absolute value distributions for the MLP projection weight matrix at layer 8 (shape $768\times 3072$).}
    \label{fig:matrix_h8_mlp}
\end{figure}

\begin{figure}[htbp]
  \centering
  \includegraphics[width=0.95\linewidth]{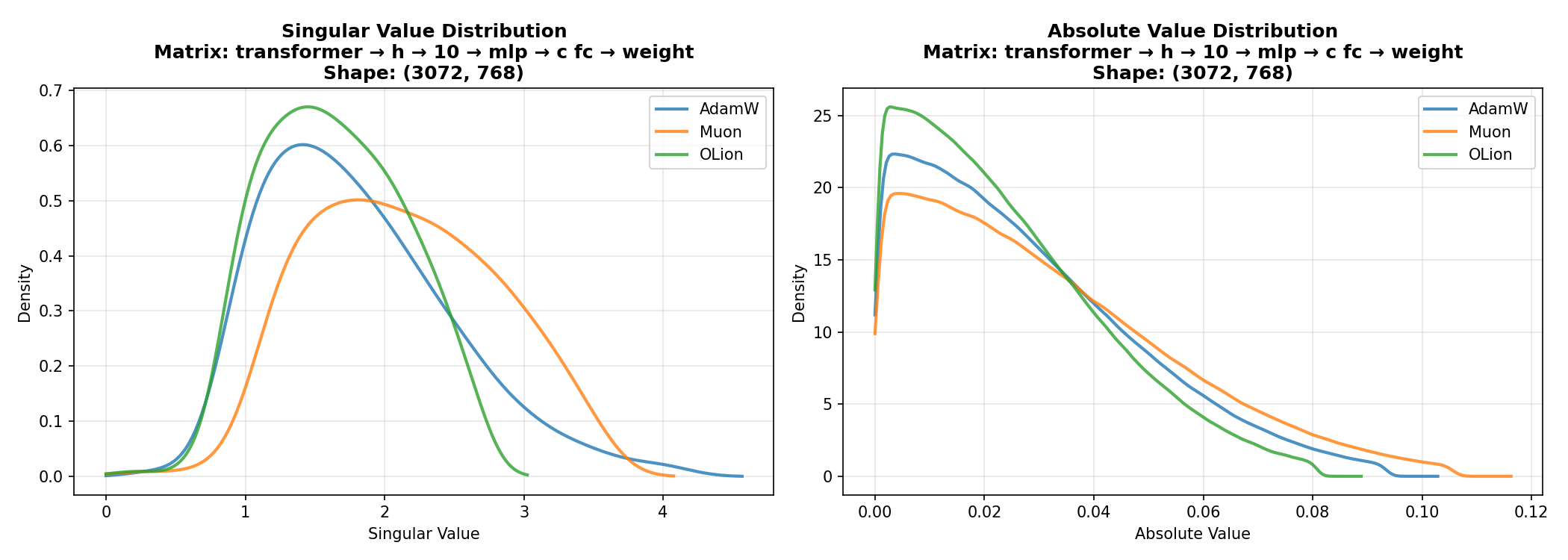}
  \caption{Singular value and absolute value distributions for the MLP projection weight matrix at layer 10 (shape $3072\times 768$).}
  \label{fig:matrix_h10_mlp}
\end{figure}

\begin{figure}[htbp]
    \centering
    \includegraphics[width=0.95\linewidth]{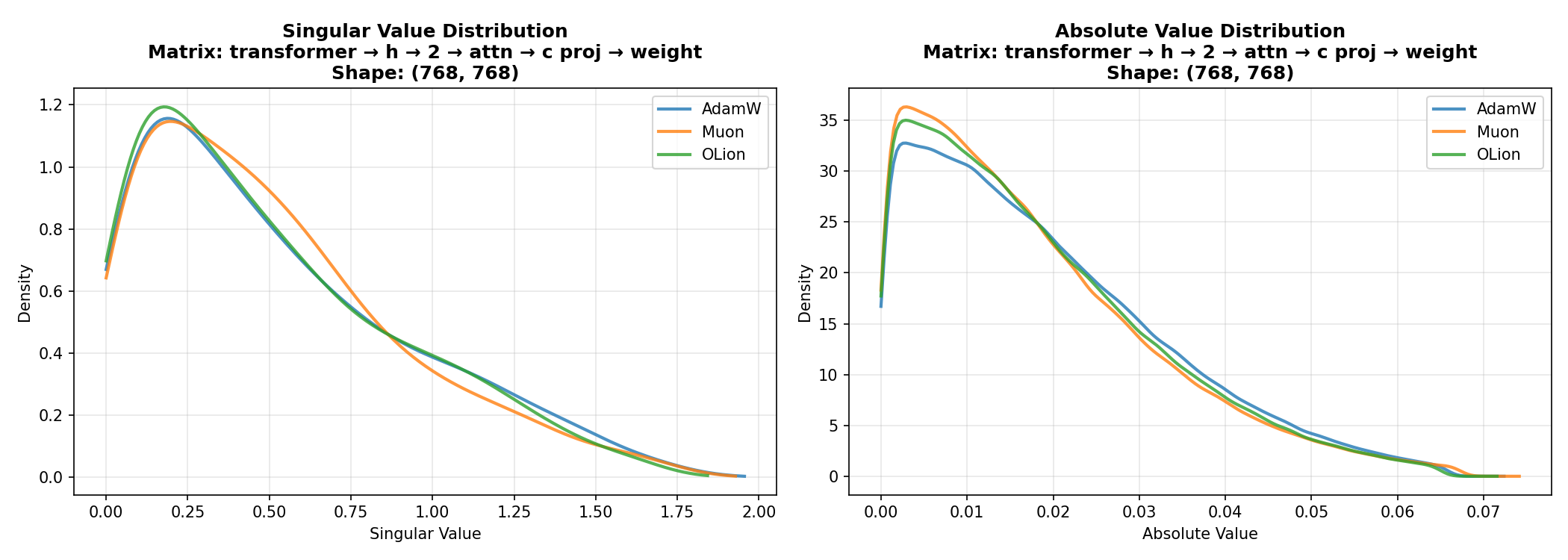}
    \caption{Singular value and absolute value distributions for the attention projection weight matrix at layer 2 (shape $768\times 768$).}
    \label{fig:matrix_h2_attn}
\end{figure}

\section{Llama-3.1-8B Supervised Fine-Tuning Details}
\label{app:sft}

We perform full parameter fine-tuning on Llama-3.1-8B model that was pretrained with AdamW, using the codebase of torchtitan~\cite{liang2025torchtitan}. We do not use LoRA or other parameter-efficient fine-tuning methods. The training dataset is MathInstruct, formatted in the standard instruction-following format for mathematical reasoning tasks. We use a learning rate of $3 \times 10^{-5}$, which is typical for supervised fine-tuning, with a linear warmup for 100 steps followed by cosine decay that gradually reduces learning rate to $3 \times 10^{-6}$. The trainer is initialized with a local batch size of 8 per GPU, global batch size of 512, and gradient accumulation over 16 steps. The sequence length is set to 2048 tokens. We train for a total of 1533 steps. Weight decay is set to $0.1$ and gradient clipping is applied at 1.0. All training is conducted in bfloat16 mixed precision.

We evaluate the fine-tuned models on five mathematical reasoning benchmarks: GSM8K (grade school math word problems) in both 0-shot and 4-shot settings, MATH (competition-level math problems) in 0-shot setting, NumGLUE (numerical reasoning benchmark) in both 0-shot and 4-shot settings, SimulEq (simultaneous equations) in 0-shot setting, and Aqua (arithmetic reasoning) in 0-shot setting. We use standard evaluation scripts for each benchmark. For optimizer hyperparameters, AdamW uses $\beta_1=0.9$, $\beta_2=0.95$, and $\epsilon=10^{-8}$. Muon uses $\beta=0.95$ with Newton-Schulz steps $K=5$. OLion uses $\beta_1=0.95$, $\beta_2=0.98$, and Newton-Schulz steps $K=5$.

Figure~\ref{fig:sft_results} reports SFT loss trajectories. \nameA{} consistently improves over Muon and is competitive
with (and in some runs better than) AdamW throughout training. 


\begin{figure}[htbp]
    \centering
        \includegraphics[width=0.75\linewidth]{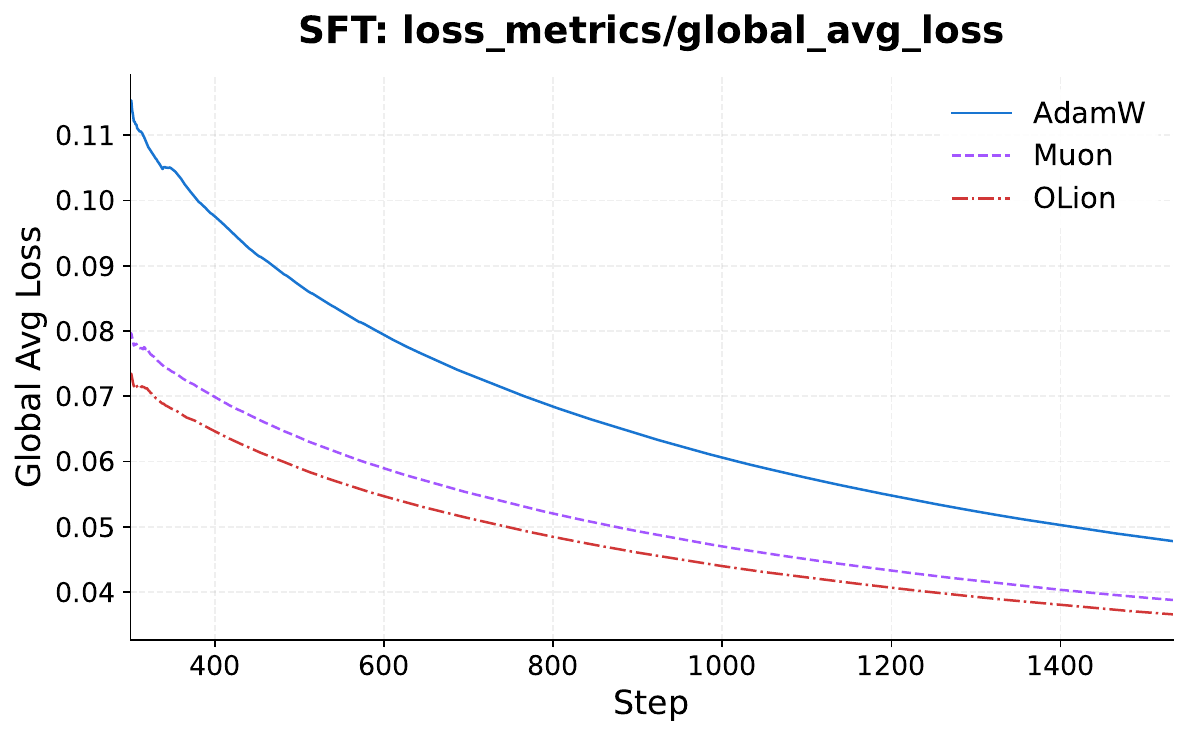}
    \caption{Supervised fine-tuning loss on Llama-3.1-8B.}
    \label{fig:sft_results}
\end{figure}


\end{document}